\documentclass[sigconf, nonacm]{acmart}

\usepackage{enumitem,kantlipsum}
\usepackage{amsmath}
\usepackage{multirow}
\usepackage{multicol}
\usepackage{graphicx}
\usepackage{subfigure}
\usepackage{caption}
\usepackage[linesnumbered,ruled]{algorithm2e}
\usepackage{balance}
\usepackage{makecell}

\newtheorem{example}{Example}
\newtheorem{lemma}{Lemma}
\newtheorem{theorem}{Theorem}

\theoremstyle{definition}
\newtheorem{definition}{Definition}

\newtheorem{prule}{Rule}

\newcommand{\blue}[1]{\textcolor{blue}{#1}}

\newcommand\vldbdoi{XX.XX/XXX.XX}
\newcommand\vldbpages{XXX-XXX}
\newcommand\vldbvolume{14}
\newcommand\vldbissue{1}
\newcommand\vldbyear{2020}
\newcommand\vldbauthors{\authors}
\newcommand\vldbtitle{\shorttitle} 
\newcommand\vldbavailabilityurl{https://github.com/Mateng0228/VPlatoon}
\newcommand\vldbpagestyle{plain} 

\begin{document}
\title{Mining Platoon Patterns from Traffic Videos}

\author{Yijun Bei}
\affiliation{%
  \institution{Zhejiang University, China}
}
\email{beiyj@zju.edu.cn}

\author{Teng Ma}
\affiliation{%
  \institution{Zhejiang University, China}
}
\email{mt0228@zju.edu.cn}

\author{Dongxiang Zhang}
\affiliation{%
  \institution{Zhejiang University, China}
  }
\email{zhangdongxiang@zju.edu.cn}

\author{Sai Wu}
\affiliation{%
  \institution{Zhejiang University, China}
}
\email{wusai@zju.edu.cn}

\author{Kian-Lee Tan}
\affiliation{%
  \institution{National University of Singapore}
}
\email{tankl@comp.nus.edu.sg}

\author{Gang Chen}
\affiliation{%
  \institution{Zhejiang University, China}
}
\email{cg@zju.edu.cn}

\begin{abstract}
Discovering co-movement patterns from urban-scale video data sources has emerged as an attractive topic. \blue{This task aims to identify groups of objects that travel together along a common route, which offers effective support for government agencies in enhancing smart city management.} However, the \blue{previous} work has made a strong assumption on the accuracy of recovered trajectories from videos and their co-movement pattern definition requires the group of \blue{objects} to appear across \blue{consecutive cameras along the common route. In practice, this often leads to} missing patterns if a vehicle is not correctly identified from a certain camera due to object occlusion or vehicle mis-matching.

To address this challenge, we propose a relaxed definition of co-movement patterns from video data, \blue{which removes the consecutiveness requirement in the common route and accommodates a certain number of missing captured cameras for objects within the group. Moreover, a novel enumeration framework called MaxGrowth is developed to efficiently retrieve the relaxed patterns.} Unlike previous filter-and-refine frameworks \blue{comprising both candidate enumeration and subsequent candidate verification procedures}, MaxGrowth incurs no verification cost for the candidate patterns. \blue{It treats the co-movement pattern as an equivalent sequence of clusters, enumerating candidates with increasing sequence length while avoiding the generation of any false positives.} Additionally, we also propose two \blue{effective} pruning rules to efficiently filter the non-maximal patterns. Extensive experiments are conducted to validate the efficiency of MaxGrowth and the quality of its generated co-movement patterns. Our MaxGrowth runs up to two orders of magnitude faster than the baseline algorithm. It also demonstrates high accuracy in real video dataset when the trajectory recovery algorithm is not perfect.
\end{abstract}

\maketitle

\pagestyle{\vldbpagestyle}
\begingroup\small\noindent\raggedright\textbf{PVLDB Reference Format:}\\
\vldbauthors. \vldbtitle. PVLDB, \vldbvolume(\vldbissue): \vldbpages, \vldbyear.\\
\href{https://doi.org/\vldbdoi}{doi:\vldbdoi}
\endgroup
\begingroup
\renewcommand\thefootnote{}\footnote{\noindent
This work is licensed under the Creative Commons BY-NC-ND 4.0 International License. Visit \url{https://creativecommons.org/licenses/by-nc-nd/4.0/} to view a copy of this license. For any use beyond those covered by this license, obtain permission by emailing \href{mailto:info@vldb.org}{info@vldb.org}. Copyright is held by the owner/author(s). Publication rights licensed to the VLDB Endowment. \\
\raggedright Proceedings of the VLDB Endowment, Vol. \vldbvolume, No. \vldbissue\ %
ISSN 2150-8097. \\
\href{https://doi.org/\vldbdoi}{doi:\vldbdoi} \\
}\addtocounter{footnote}{-1}\endgroup

\ifdefempty{\vldbavailabilityurl}{}{
\vspace{.3cm}
\begingroup\small\noindent\raggedright\textbf{PVLDB Artifact Availability:}\\
The source code, data, and/or other artifacts have been made available at \url{\vldbavailabilityurl}.
\endgroup
}

\section{Introduction}
Co-movement pattern mining from large-scale GPS trajectories has been extensively studied over the past two decades~\cite{BENKERT2008report-flock, Jeung2008pvldb, li2010swarm, fan2016spare}. Recently, this mining task was applied to video data to aid government agencies in smart city management, including detecting traffic congestion at various granularities~\cite{orakzai2019k} and monitoring suspicious groups for security~\cite{yadamjav2020recurrent}. The underlying motivation is that trajectory data are primarily collected and owned by commercial hail-riding or map-service companies, but not directly accessible to government agencies. In contrast, government-installed surveillance cameras have covered the entire urban area, \blue{making it advantageous} to re-examine the problem of co-movement pattern mining from videos.

\begin{figure}[h!]
    \centering
	\includegraphics[width=0.44\textwidth]{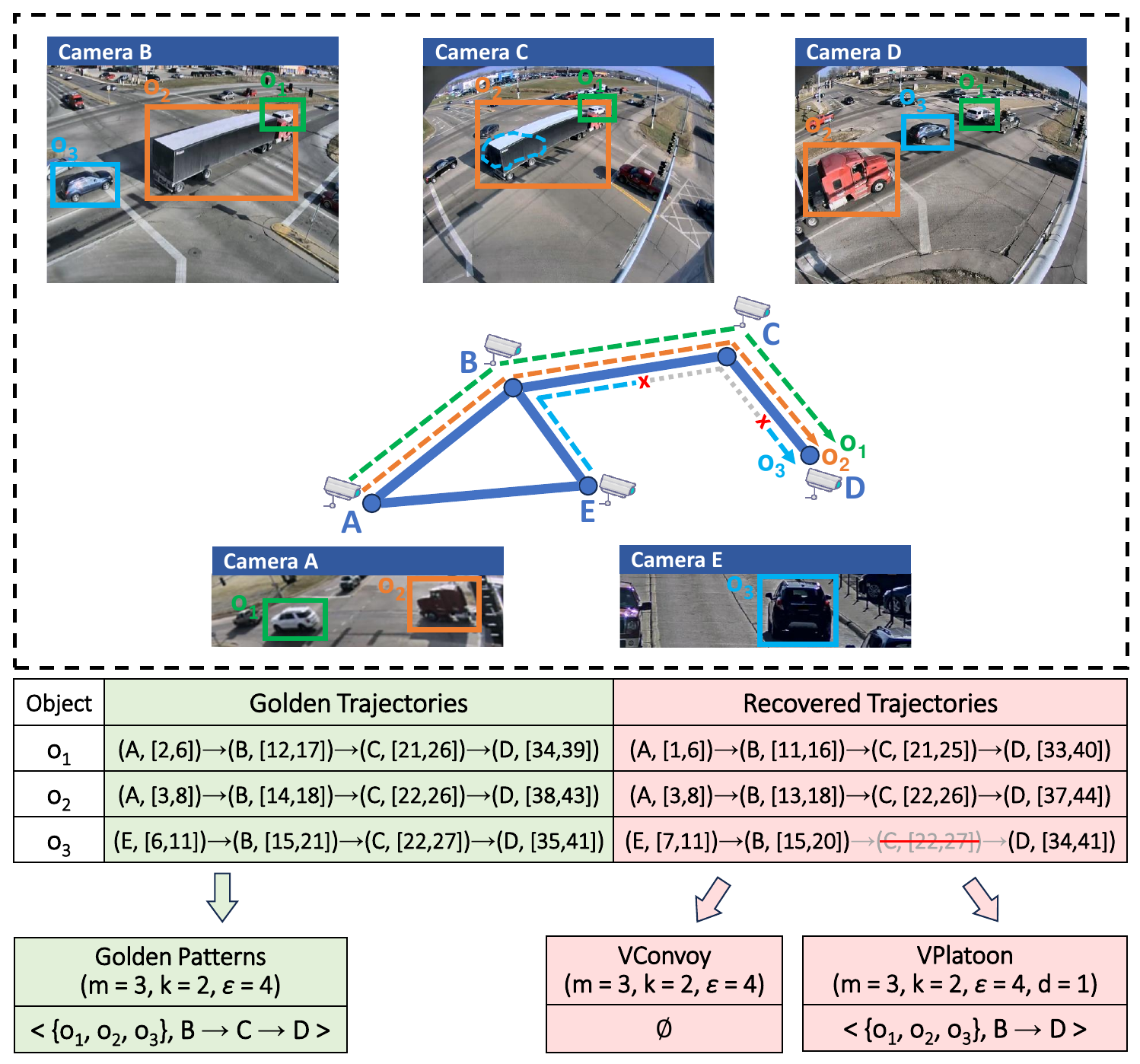}
    \caption{An illustrative example for relaxed co-movement pattern mining from real video data.}
    \label{fig:problem-example}
    \vspace{-2mm}
\end{figure}

In the scenario of video mining, \blue{the trajectory of each individual object} is extracted using cross-camera trajectory recovery algorithms~\cite{he2020citytrack, tong2021recst, yang2022box} and represented as a sequence of camera IDs and time intervals. \blue{As an example, Figure~\ref{fig:problem-example} depicts the golden (ground-truth) trajectories and recovered trajectories of three objects $o_1$, $o_2$, and $o_3$. Since the GPS trajectory-based pattern definition is unsuitable for video data}, the co-movement pattern is redefined as a group of objects traveling together along a common route in the road network~\cite{Zhang2023pvldb}. \blue{Specifically, this definition involves three parameters: $m$, $k$, and $\epsilon$. That is, each valid group must contain at least $m$ objects; these objects must traverse at least $k$ consecutive cameras which constitute the common route; and the objects must be captured by each common camera within a time interval of $\epsilon$. We call this pattern VConvoy, as it can be regarded as a migration of the GPS trajectory-based pattern Convoy~\cite{jeung2008convoy}. In Figure~\ref{fig:problem-example}, three objects travel together across cameras $B$, $C$ and $D$. If $m=3$, $k=2$ and $\epsilon=4$, a valid VConvoy pattern $\langle\{o_1,o_2,o_3\}, B \rightarrow C \rightarrow D \rangle$ can be mined from the golden trajectories.}

\blue{However, the trajectory recovery algorithms in practice are not perfect, and the available recovered trajectories often exhibit differences from the golden trajectories. On the other hand, VConvoy assumes no accuracy loss in the mining trajectories and requires that the objects in a valid pattern appear across a consecutive sequence of cameras to constitute the common route. As a result, this rigid setting incurs the risk of pattern missing. For instance, in Figure~\ref{fig:problem-example}, vehicle $o_3$ is occluded by truck $o_2$ in camera $C$, resulting in a recovered trajectory of $E \rightarrow B \rightarrow D$ that omits camera $C$. Hence, when applying VConvoy pattern mining to the recovered trajectories, the pattern $\langle\{o_1,o_2,o_3\},B\rightarrow C\rightarrow D\rangle$, which is identifiable from the golden trajectories, will be absent.} Additionally, object occlusion is not the sole issue; object ID switching poses a significant challenge for existing object tracking algorithms~\cite{dendorfer2020mot20, sun2022dancetrack, wang2022split}. If moving objects are incorrectly identified and assigned wrong IDs, VConvoy mining will also fail under such circumstances.

To effectively support co-movement mining from imperfectly recovered trajectories, we propose a new definition of a relaxed pattern in this paper. Inspired by the platoon pattern~\cite{li2015platoon}, we remove the requirement for consecutive camera sequences and introduce a parameter $d$, which allows for a certain number of missing cameras \blue{in the common route} of objects within the group. \blue{We call this new pattern VPlatoon. In other words, VPlatoon involves four parameters: $m$, $k$, $\epsilon$, and $d$. It still requires a valid group of at least $m$ objects to exhibit temporal proximity $\epsilon$ along a common route of length at least $k$, but the objects in the group now can briefly leave the common route with a maximum camera gap of $d$. For example, the common route among the recovered trajectories of $\{o_1,o_2,o_3\}$ in Figure~\ref{fig:problem-example} is $B \rightarrow D$, bypassing a missing camera $C$. With $d=1$, the pattern $\langle\{o_1,o_2,o_3\},B\rightarrow D\rangle$ becomes a valid VPlatoon according to our relaxed pattern definition.}

After eliminating the constraint of a consecutive sequence of cameras, the search space is expanded and the problem becomes more complex than the original \blue{VConvoy mining problem}, which has been proved to be NP-Hard in~\cite{Zhang2023pvldb}. We present a baseline algorithm that extends the TCS-tree algorithm~\cite{Zhang2023pvldb}. \blue{This baseline builds on the basic filter-and-refine framework of the TCS-tree with two key improvements. First, in the filter stage, we substitute the original frequent substring miner in TCS-tree with the sequential pattern miner~\cite{yan2003clospan, wang2004bide} and further improve filter effectiveness based on the idea of camera partitioning. Second, in the refinement stage, we incorporate multiple buffers to efficiently accommodate the gap tolerance parameter $d$.}

According to our empirical analysis, the baseline suffers from high candidate verification cost because the candidate search space is huge and a valid candidate requires multiple constraints to be satisfied simultaneously. In this paper, we propose a novel candidate enumeration framework called MaxGrowth. It eliminates the necessity of the expensive verification procedure without generating any false positives. Thus, compared with the baseline, it incurs no verification cost and the total mining time can be remarkably reduced. The core idea is to \blue{treat a VPlatoon pattern} as an equivalent sequence of clusters, \blue{where each cluster comprises a specific camera and a group of proximate objects. Patterns are then constructed by incrementally expanding the cluster sequences. In such a manner,} the pattern mining task is decomposed into a series of feasible cluster selection problems, \blue{and MaxGrowth can effectively utilize pattern constraints during enumeration while avoiding the need to handle numerous intermediate false-positive candidates.} Moreover, to efficiently remove valid but non-maximal candidate patterns, \blue{we introduce two pruning rules: the root pruning rule and the dependency pruning rule. The former prunes non-contributing search space at the root node of the search tree, while the latter allows for more refined pruning in subsequent search nodes during candidate enumeration.}

To sum up, the key contributions of this paper are as follows.
\begin{itemize}
    \item We present a new type of video-based co-movement pattern mining problem, whose goal is to be more tolerant with flawed video tracking and trajectory recovery algorithms. 
    \item We present a novel candidate enumeration framework called MaxGrowth that eliminates the necessity of candidate validness verification cost. 
    \item Two effective pruning rules are proposed to remove non-maximal candidate patterns, which can reduce the cost of dominance verification by over 90\%.
    \item Extensive experiments are conducted on real-world datasets, validating the effectiveness of the proposed pattern while demonstrating the efficiency and scalability of our mining techniques.
\end{itemize}

The rest of the paper is organized as follows. We review related literature in Section~\ref{sec:related-work}. The basic concepts and problem statement are presented in Section~\ref{sec:definition}. Section~\ref{sec:baseline} presents the baseline algorithm. we propose our enumeration framework MaxGrowth in Section~\ref{sec:cls-growth} and the two pruning rules for maximal pattern mining are presented in Section~\ref{sec:max-pattern}. Experimental evaluation is conducted in Section~\ref{sec:experiment}. We conclude the paper in Section~\ref{sec:conclusion}.

\section{RELATED WORK}\label{sec:related-work}

Co-movement pattern mining intends to find a group of objects moving within spatial proximity over a specified time interval.  Previous studies can be divided into two categories according to data sources: GPS trajectory-based and video-based.

In the realm of GPS trajectory-based co-movement patterns, the flock~\cite{BENKERT2008report-flock,gudmundsson2006computing} and convoy~\cite{jeung2008convoy} patterns both require that candidate groups appear across consecutive timestamps. The primary distinction between them lies in their definition of spatial proximity. The flock pattern requires objects within the same cluster to be within a disk with a diameter smaller than a specified parameter, whereas the convoy pattern employs density-based spatial clustering~\cite{ester1996dbscan}.
The performance bottleneck in mining these two types of co-movement patterns is the clustering overhead applied at each timestamp. Besides general acceleration methods in spatiotemporal management~\cite{chen2015efficient,wang2019fast,ding2018ultraman}, several specialized techniques for co-movement mining have been developed to mitigate this issue, such as trajectory simplification~\cite{Jeung2008pvldb}, spatial partitioning~\cite{liu2021ecma,chen2019real}, and the divide-and-conquer scheme~\cite{orakzai2019k}.

In the \blue{definitions} of group~\cite{wang2006group}, swarm~\cite{li2010swarm}, and platoon~\cite{li2015platoon} patterns, the temporal duration constraint is relaxed. Their mining algorithms follow similar solving schemes, where the main idea is to grow an object set from an empty set. Meanwhile, various pruning techniques were designed to ensure mining efficiency. The group pattern mining primarily utilizes its proposed VG-graph structure. While the platoon pattern miner applies multiple fine-grained pruning rules based on the prefix table. On the other hand, the swarm pattern mining algorithm mainly introduces two pruning rules called backward and forward pruning to trim the search space for mining closed swarm patterns. These optimization techniques are based on the analysis of time dimension, but in our scenario, the time information is restricted to independent time intervals, which hinders their adaptation to our problem.

Video-based co-movement pattern mining was first examined by Zhang et al.~\cite{Zhang2023pvldb}. The problem was proven to be NP-hard. To solve it efficiently, an index called temporal-cluster suffix tree (TCS-tree) and a sequence-ahead pruning framework based on TCS-tree were proposed. To reduce verification cost on the candidate paths, the authors introduced a sliding-window based co-movement pattern enumeration strategy and a hashing-based dominance eliminator. Our work extends the definition to address the issue of missing patterns incurred by object occlusion or vehicle mis-matching.


\section{Problem Definition}\label{sec:definition}

Given a corpus of input video data captured by surveillance cameras in scenes such as urban areas or highways. We denote all objects appearing in the input video as $\mathbb{O} = \{ o_1,o_2,o_3\dots\}$.
Following previous definition of co-movement pattern mining~\cite{Zhang2023pvldb}, we assume that the travel paths of these objects can be extracted via trajectory recovery algorithms~\cite{tong2021recst, he2020citytrack} as a pre-processing step.

\begin{figure}[h!]
    \centering
    \includegraphics[width=0.275\textwidth]{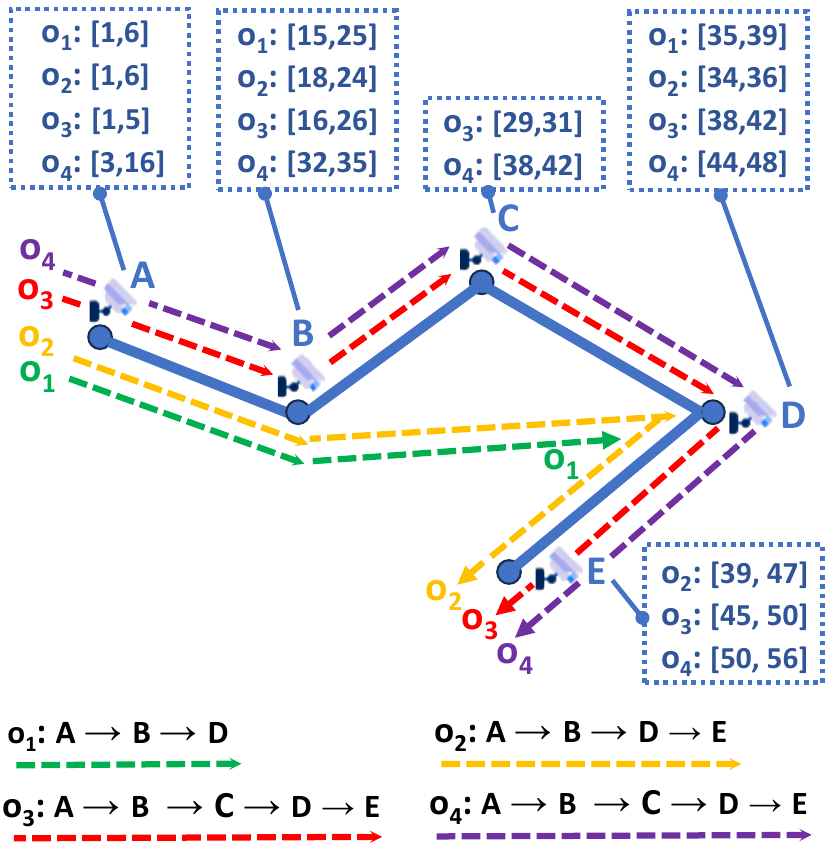}
    \caption{Travel paths for four objects $\{o_1,o_2,o_3,o_4\}$.}
    \label{fig:run-example}
    \vspace{-3mm}
\end{figure}

\begin{definition} \textbf{Travel Path} \\
The travel path $P_i$ of an object $o_i \in \mathbb{O}$ is defined as a sequence of surveillance cameras with associated time intervals:\\
$P_i = (c_1, [s_1, e_1]) \rightarrow (c_2, [s_2, e_2]) \rightarrow \ldots \rightarrow (c_n, [s_n, e_n])$\\
where $(c_j, [s_j, e_j])$ indicates that $o_i$ is captured by camera $c_j$ during time interval $[s_j, e_j]$ with $s_j < e_j$ and $s_j < s_{j+1}$.
\end{definition}

\blue{We call $s_j$ and $e_j$ the entrance and exit timestamp of $o_i$ at camera $c_j$, respectively, and the time interval of $P_i$ is $[s_i, e_n]$.} Moreover, when there is no need to mentioned the temporal dimension, we abbreviate $P_i$ as $c_1 \rightarrow c_2 \rightarrow \ldots \rightarrow c_n$ and let $c^i_j$ denote the $j$-th camera \blue{$(c_j)$} that object $o_i$ passes through.

\begin{example}
Figure~\ref{fig:run-example} illustrates the travel paths of four example objects, where the road network is represented by thick solid lines and the extracted travel paths are represented by dashed lines.
The travel path $P_1$ of object $o_1$ is $(A, [1, 6]) \rightarrow (B, [15, 25]) \rightarrow (D, [35, 39])$, and the time interval of $P_1$ is $[1, 39]$.
\end{example}

To represent the subpath relationships between imperfectly recovered trajectories and suppoert our relaxed definition of co-movement pattern, we define the notion of $d$-subpath: 

\begin{definition} $\mathbf{d}$\textbf{-subpath} \\
The travel path $P_i$ is a $d$-subpath of $P_j$, if the time interval of $P_i$ is contained in $P_j$ and there exists an injection function $f:\{1,2,\ldots,|P_i|\} \rightarrow \{1,2,\ldots,|P_j|\}$ satisfying:
\begin{itemize}
    \item $\forall 1 \leq t < |P_i|$, $f(t+1) - f(t) \leq d + 1$.
    \item $\forall 1 \leq t \leq |P_i|$, $c^i_t = c^j_{f(t)}$.
\end{itemize}
\end{definition}

The gap tolerance parameter $d$ is used to controls the number of missing cameras between adjacent cameras when mapping $P_i$ to $P_j$. Specifically, given that the time interval of $P_i$ is contained within $P_j$, $P_i$ exactly matches a consecutive substring of $P_j$ when d = 0. \blue{In contrast}, $P_i$ can be any subsequence of $P_j$ when d = $\infty$.

\begin{example}
In Figure~\ref{fig:run-example}, when $d = 1$, $P_1$ is a $d$-subpath of $P_3$. First, the time interval of $P_1$ ($[1, 39]$) is contained in the time interval of $P_3$ ($[1, 50]$). Additionally, given the mapping $f: \{1, 2, 3\} \to \{1, 2, 4\}$, we find the corresponding consistent cameras in $P_3$ for each camera in $P_1$ according to $f$: $c_1^{1} = c_1^{3}$, $c_2^{1} = c_2^{3}$, and $c_3^{1} = c_4^{3}$. Furthermore, for the mapped cameras in $P_3$, the distances between adjacent cameras are no more than 2, i.e., $c_4^{3} - c_2^{3} \leq 2$ and $c_2^{3} - c_1^{3} \leq 2$.
\end{example}

We continue to define the notion of $\epsilon$-close to indicate the proximity between objects. Here, $\epsilon$-close implies the temporal closeness of objects passing through the same camera.

\begin{definition}\label{def:eps} $\mathbf{\epsilon}$\textbf{-close at camera} $\mathbf{c}$\\
A group of objects $O \subseteq \mathbb{O}$ is $\epsilon$-close at camera $c$, if $\forall$ objects $o_i, o_j \in O$, the gap between their entrance timestamps at $c$ (i.e., $|s_i - s_j|$) does not exceed $\epsilon$.
\end{definition}

If an object $o$ belongs to a group of $\epsilon$-close objects $O$, we also call that $o$ is $\epsilon$-close to any other objects in $O$.

\begin{example}
As shown in Figure~\ref{fig:run-example}, When $\epsilon = 6$, we can say that the object set $\{o_1, o_2, o_3\}$ is $\epsilon$-close at camera $B$, as the time gap between the entrance timestamps of any object pair is less than $\epsilon$. Specifically, $|s_1 - s_2| = 3 \leq 6$, $|s_1 - s_3| = 1 \leq 6$, and $|s_2 - s_3| = 2 \leq 6$. In contrast, the object set $\{o_1, o_2, o_3, o_4\}$ is not $\epsilon$-close at camera $B$ because $|s_1 - s_4| = 17 > 6$.
\end{example}

Now, we formalize the definition of the relaxed co-movement pattern from videos \blue{ (VPlatoon)}, which represents a group of objects traveling together satisfying $\epsilon$-close \blue{along} a common $d$-subpath.

\begin{definition} \textbf{The Relaxed Co-movement Pattern} \\
For parameters $m$, $k$, $d$ and $\mathbf{\epsilon}$, the relaxed co-movement pattern is defined as $R = \langle O, P \rangle$, where $O$ and $P$ represent a set of objects and a common path respectively. Meanwhile, it must satisfy the following conditions:
\begin{enumerate}
    \item $O$ is $\epsilon$-close at each camera of $P$.
    \item For each object $o_i \in O$, $P$ is a $d$-subpath of $P_i$.
    \item $|O| \geq m$.
    \item $|P| \geq k$.
\end{enumerate}
\end{definition}

\begin{example}
In Figure~\ref{fig:run-example}, given $m = 2$, $k = 3$, $d = 1$, and $\epsilon = 6$, $\langle \{o_1, o_2, o_3\}, A \to B \to D \rangle$ is a relaxed co-movement pattern. The object set \blue{$\{o_1, o_2, o_3\}$} contains more than 2 objects and is $\epsilon$-close at cameras $A$, $B$ and $D$. Meanwhile, the common route $A \to B \to D$ has a length of at least 3 and is a $d$-subpath of $P_1$, $P_2$, and $P_3$.
\end{example}

Those relaxed co-movement patterns with a large number of objects or a long common route could possess a combinatorial number of redundant "sub-patterns" that do not convey more information than their parent pattern. For the sake of conciseness, we define the concept of maximal pattern.

\begin{definition}\label{def:max-pattern} \textbf{Maximal Relaxed Co-movement Pattern}\\
A relaxed co-movement pattern $R_i=\langle O_i,P_i \rangle$ is maximal if there exists no other pattern $R_j = \langle O_j, P_j \rangle$ such that $O_i \subseteq O_j$ and $P_i$ is a $d$-subpath of $P_j$.
\end{definition}

Our goal is to discover all the maximal relaxed co-movement patterns under parameters $m$, $k$, $d$ and $\epsilon$.
\begin{example}
In Figure~\ref{fig:run-example}, for $m = 2$, $k = 3$, $d = 1$, and $\epsilon = 6$, $R = \langle \{o_1, o_2, o_3\}, A \to B \to D \rangle$ is a maximal pattern, whereas $R_i = \langle \{o_2, o_3\}, A \to D \rangle$ is not a maximal pattern. This is because the object set \(\{o_2, o_3\}\) in $R_i$ is a subset of \(\{o_1, o_2, o_3\}\), and the common path $A \to D$ in $R_i$ is a $d$-subpath of $A \to B \to D$. In contrast, there is no other pattern for $R$ that satisfies above similar conditions.
\end{example}

Table~\ref{tbl:notation} provides a summary of the frequently used notations.
\begin{table}[ht!]
\centering
\vspace{-2mm}
\caption{Notation Table.}
\vspace{-2mm}

\resizebox{0.9\linewidth}{!}{
\begin{tabular}{|c|l|} \hline
$\mathbb{O}$ & All moving objects contained in $\mathbb{V}$\\ \hline
$P_i$ & The travel path of an object $o_i \in \mathbb{O}$\\ \hline
$R$ & The relaxed co-movement pattern from video data\\ \hline
$m$ & The minimum number of objects in $R$\\ \hline
$k$ & The minimum length of the common route in $R$\\ \hline
$d$ & The gap tolerance in $R$\\ \hline
$\epsilon$ & The threshold of closeness in $R$\\ \hline
$R.O\text{/}P$ & The object set / common route of $R$\\ \hline
$CL_i=(O_i,c_i)$ & A certain cluster with object set $O$ derived from camera $c$\\ \hline
$p^i_j$ & The position of camera $CL_i.c$ in object $o_j$'s travel path $P_j$\\ \hline
$S$ & A cluster sequence \\ \hline
$\lambda(S)$ & The core objects of $S$\\ \hline
$R(S)$ & The candidate co-movement pattern corresponding to $S$\\ \hline
\end{tabular}
}

\label{tbl:notation}
\vspace{-3mm}
\end{table}

\section{Baseline Algorithm}\label{sec:baseline}
In this section, we propose a baseline algorithm that extends TCS-tree~\cite{Zhang2023pvldb} for VPlatoon pattern mining. It adopts the \blue{filter-and-refine} framework similar to TCS-tree and we call this baseline FRB algorithm. \blue{We begin with a brief introduction to TCS-tree, then provide the detailed presentation of the FRB.}

\blue{The TCS-tree algorithm is a state-of-the-art VConvoy mining algorithm that follows the filter-and-refine framework. Its filter stage relies on an efficient index, the temporal-cluster suffix tree, which performs two-level temporal clustering within each camera and constructs a suffix tree from the resulting clusters. By adopting frequent substring mining based on this index, TCS-tree identifies common subsequences as candidates through low granularity filtering. Each subsequence is a consecutive camera sequence with a length of at least $k$ and support from at least $m$ proximate objects. In the subsequent refinement stage, TCS-tree uses a modified CMC algorithm~\cite{Jeung2008pvldb, kalnis2005movingcls} to precisely verify the temporal proximity (parameter $\epsilon$) of objects along the common subsequence for each candidate and obtain the final results.}

In order to extend TCS-tree as a tailored algorithm for VPlatoon pattern mining, we introduce two key improvements. First, in the filter stage, we replace the original frequent substring miner with the sequential pattern miner~\cite{yan2003clospan, wang2004bide} for candidate enumeration and partition each camera based on temporal proximity to improve enumeration effectiveness. Second, in the refinement stage, we incorporate multiple buffers into the original candidate verification algorithm to accommodate the gap tolerance parameter $d$.

\blue{We now present the filter stage in detail. In this stage, FRB focuses on the route information within the input travel paths and identifies frequent camera subsequences that represent common routes as candidates.} Specifically, each candidate subsequence must contain at least $k$ cameras and appear in the travel paths of at least $m$ objects. Moreover, since the definition of VPlatoon removes the consecutive constraint on the common route, candidate subsequences can be traversed both consecutively and non-consecutively. Therefore, the basic scheme of FRB in the filter stage is to first remove the time intervals from each moving object's travel path and represent the path as a sequence of cameras in ascending order of their entrance times. Then, the sequential pattern mining algorithm~\cite{yan2003clospan, wang2004bide} with the minimum length of $k$ and support of $m$ is applied to these camera sequences for candidate enumeration.

Furthermore, since the sequential pattern mining algorithm cannot be accelerated by the TCS-tree index, we adapt the core idea of this index to partition the cameras and generate a more nuanced sequence representation. Performing sequential pattern mining on these nuanced sequences instead of the camera sequences, FRB could filter out more false positives in the filter stage.

Specifically, we partition each camera based on the temporal proximity of the objects passing through it. The partitions of a camera $c$ is defined as $\mathbb{T}_c = \{T^1_c, T^2_c, ..., T^n_c \}$, where:
\begin{itemize}
    \item For $1 \leq i \leq n$, $T^i_c$ is a set of objects passing through $c$.
    \item $\bigcup_{i=1}^{n} T^i_c$ constitutes all the objects passing through $c$.
    \item $\forall o_u \in T^i_c$, if there exists $o_v$ that $o_u$ is $\epsilon$-close to $o_v$, then $o_v$ is also in $T^i_c$; otherwise, $o_u$ is the only element in $T^i_c$.
\end{itemize}

It's not difficult to see that an object belongs to exactly one partition when passing through a \blue{given} camera. We can transform each object's camera sequence into a  partition sequence corresponding to the cameras it passes through. Two objects might have a common camera subsequence, but their partition sequences can be completely different if they are not $\epsilon$-close. Therefore, this fine-grained representation offers stronger candidate filtering capabilities.

\begin{example}
In Figure~\ref{fig:run-example}, let $\epsilon = 6$. Since all objects at camera $A$ are $\epsilon$-close, $\mathbb{T}_A = \{ \{o_1, o_2, o_3, o_4\} \}$. For camera $B$, we have $o_1, o_2, o_3$ are $\epsilon$-close but $o_4$ is not $\epsilon$-close to them, so $\mathbb{T}_B = \{ \{o_1, o_2, o_3\}, \{o_4\} \}$. Similarly, $\mathbb{T}_C = \{ \{o_3\}, \{o_4\} \}$.
If we only consider cameras $A$, $B$ and $C$, objects $o_3$ and $o_4$ share the same sequence of camera $A \to B \to C$. However, their partition sequences are $T^1_A \to T^1_B \to T^1_C$ and $T^1_A \to T^2_B \to T^2_C$, respectively. Mining on these two partitions prevents the generation of false positive $(A \to B \to C, \{o_3, o_4\})$.
\end{example}

The following lemma is presented to demonstrate the completeness of FRB in the filter stage.
\begin{lemma}\label{lm:frb}
Each valid VPlatoon pattern is contained in at least one candidate from the filter stage.
\end{lemma}
\begin{proof}
We proof by contradiction. Suppose there exists a missing valid pattern $R = \langle O, P \rangle$. This implies that the objects in $O$ is not always in the same partition at each camera of $P$. So there exists a camera $c$ in $P$ where $O$ belongs to at least two different partitions of $c$. Since objects that are $\epsilon$-close must belong to the same partition, $O$ is not $\epsilon$-close at $c$. However, we know that $O$ must be $\epsilon$-close at each camera of $P$ by definition, which leads to a contradiction.
\end{proof}

\blue{We proceed to describe the refinement stage,} where FRB refines each candidate $(Seq, O)$ to ensure the temporal proximity $(\epsilon)$ and gap tolerance $(d)$ of objects in $O$ as they travel along the common camera subsequence $Seq$. The approach is similar to the candidate verification algorithm in TCS-tree, with additional buffer structures introduced to improve efficiency in adapting to the VPlatoon definition.
Specifically, the algorithm first computes all groups of objects in $O$ that are $\epsilon$-close for each camera of $Seq$. It then searches for valid patterns along $Seq$ while maintaining a global buffer at each camera to store partial patterns that successfully reach it. During the search at each camera, the groups of $\epsilon$-close objects are intersected with the partial patterns stored in the buffers of the previous $d+1$ cameras in $Seq$, thereby generating new VPlatoon patterns.

Algorithm~\ref{alg:frb} demonstrates the FRB algorithm. During the filter stage (lines 1-4), we first retrieve the camera sequences of all objects by removing the time intervals from their travel paths (line 1). Then, we compute all the partition sequences based on our aforementioned description of partitions (lines 2-3). In the final step of the filter stage, sequential pattern mining is performed on $\mathbb{S}_{T}$ for candidate enumeration (line 4).
During the subsequent refinement stage (lines 5-19), we first initialize $\mathbb{R}$ as the final result set and $Buf$ as the global buffer set (line 5). Then, the algorithm iterates over the cameras in $Seq$ of each candidate (lines 6-7).
At each processed camera $c_i$, all groups of $\epsilon$-close objects in $O$ are computed according to Definition~\ref{def:eps} and intersected with previous partial patterns to generate new partial patterns $R_{new}$ (lines 8-11). 
At this point, it's important to note that we still need to further verify whether each object in $R_{new}.O$ has traversed more than $d$ cameras that are not in $Seq$ between the current camera $c_i$ and the last reached camera $c_j$. This is done using the distance function $distance(o, c_i, c_j)$ (lines 12-13). Afterward, we can then update $Buf$ and $\mathbb{R}$ (lines 14-17).
Finally, the algorithm derives the final results by removing all non-maximal patterns from $\mathbb{R}$. This step is similar to the one used in TCS-tree.

\begin{algorithm}[t!]
\SetAlgoNoEnd \SetAlgoNoLine 
\caption{The FRB Algorithm}\label{alg:frb}
$\mathbb{S}_{c} \leftarrow$ retrieve the camera sequences of all objects\;
$\mathbb{T} \leftarrow$ compute partitions of each camera\;
$\mathbb{S}_{T} \leftarrow$ convert \blue{$\mathbb{S}_{c}$} into partition sequences based on \blue{$\mathbb{T}$}\;
$\mathbb{R}_c \leftarrow$ sequential pattern mining on $\mathbb{S}_{T}$ with $m$ and $k$\;
$\mathbb{R} \leftarrow \emptyset$; $Buf \leftarrow \emptyset$\;
\ForEach{candidate $(Seq, O) \in \mathbb{R}_c$}{
    \ForEach{camera $c_i$ in Seq}{
        $\mathbb{O}_{\epsilon} \leftarrow$ compute groups of $\epsilon$-close objects in $O$ at $c_i$\;
        \For{camera $c_j \in \{ c_{i-1-d}, \ldots, c_{i-1} \}$}{
            \For{$(O_{\epsilon}, R_b) \in \mathbb{O}_{\epsilon} \times Buf_{c_j}$}{
                $R_{new} = \langle R_b.O \cap O_\epsilon, R_b.P \cup c_i \rangle$\;
                \For{each object $o \in R_{new}.O$}{
                    \lIf{distance($o$, $c_i$, $c_j$) $> d$}{discard $o$}
                }
                \lIf{$|R_{new}.O| < m$}{\textbf{continue}}
                $Buf_{c_i} \leftarrow Buf_{c_i} \cup \{R_{new}\}$\;
                \lIf{$|R_{new}.P| \geq k$}{$\mathbb{R} \leftarrow \mathbb{R} \cup \{R_{new}\}$}
            }
        }
        \lForEach{$O_{\epsilon} \in \mathbb{O}_{\epsilon}$}{$Buf_{c_i} \leftarrow Buf_{c_i} \cup \{ \langle O_{\epsilon}, [c_i] \rangle \}$}
    }
}
$\mathbb{R} \leftarrow$ remove non-maximal patterns from $\mathbb{R}$\;
\textbf{return} $\mathbb{R}$\;
\end{algorithm}

\section{Candidate Enumeration Framework}\label{sec:cls-growth}
The drawback of FRB algorithm is that its filter-and-refine framework incurs substantial computation overhead for candidate verification. As will be empirically examined in Section 7, the candidate verification becomes the performance bottleneck in FRB algorithm. To address the challenge, we present MaxGrowth algorithm, which eliminates the requirement for candidate verification. In this section, we first present its candidate enumeration scheme to find valid candidate patterns, and in the next section, we present the pruning rules to efficiently eliminate patterns that are non-maximal.

The core idea of MaxGrowth is to treat the co-movement pattern as a sequence of fundamental units (clusters) and iteratively grow such a sequence for pattern generation. This scheme contrasts with the traditional filter-and-refinement framework, which treats the co-movement pattern as an overall combination of objects with their common route and performs refinement for pattern generation. In other words, MaxGrowth decomposes the co-movement pattern mining problem into a series of feasible cluster selection problems during candidate enumeration. The expensive verification procedure is avoided without generating false positives. \blue{We next introduce the cluster representation of the co-movement pattern in Subsection~\ref{subsec:maxgrowth-cls} and then elaborate on the candidate enumeration scheme of MaxGrowth in Subsection~\ref{subsec:maxgrowth-scheme}.}

\subsection{Cluster Representation}\label{subsec:maxgrowth-cls}
\blue{As the fundamental unit in MaxGrowth,} a cluster is essentially an aggregation of a set of objects and a traversed camera.

\begin{definition}\label{def:cls} \textbf{Cluster} \\
A cluster is defined as $CL = (O, c)$, if $|O| \geq m$ and $O$ is $\epsilon$-close at $c$.
\end{definition}

We use $CL.O$ and $CL.c$ to denote the object set and the traversed camera of a given cluster $CL$, respectively.
Besides, it is sometimes necessary to determine the specific position of an object as it traverses the camera of a given cluster in order to analyze the positional relationships between clusters. Therefore, we use the notation $\boldsymbol{p^i_j}$ to represent the position of object $o_j$ in its travel path $P_j$ when passing through the camera in cluster $CL_i$.

\begin{example}
In the example of Figure~\ref{fig:run-example}, if we set $m = 2$ and $\epsilon = 6$, then the object set $\{o_1, o_2, o_3\}$ contains more than 2 objects and is $\epsilon$-close at camera $B$. Thus, $(\{o_1, o_2, o_3\}, B)$ can be a cluster, which is shown as $CL_2$ in Figure~\ref{fig:cluster-example}. Meanwhile, since $o_1$'s travel path is $P_1 = (A, [1, 6]) \rightarrow (B, [15, 25]) \rightarrow (D, [35, 39])$, and $o_1$ belongs to $CL_2.O$ at the second position of $P_1$, we have $p^2_1 = 2$.
\end{example}

\blue{Intuitively, a cluster describes the co-movement scene under a single camera. To further characterize the co-movement pattern along the common route, we need to link the clusters in a reasonable manner. To this end, we now present several essential definitions related to the cluster sequence.}

\begin{definition}\label{def:core-object} \blue{\textbf{Core Object}} \\
We denote object $o_c$ as a core object of the cluster sequence $S$ if:
\begin{itemize}
    \item $o_c \in \bigcap_{i=1}^{|S|} CL_i.O$.
    \item For $1 < i \leq |S|$, $0 < p^i_c - p^{i-1}_c \leq d + 1$.
\end{itemize}
\end{definition}

Here, $|S|$ represents the number of clusters (length) in the cluster sequence $S = [CL_1, \dots, CL_{|S|}]$. We further use \blue{the notation $\boldsymbol{\lambda(S)}$} to denote all the core objects in $S$. Actually, $\lambda(S)$ indicates the group of objects that is $\epsilon$-close along a common $d$-subpath constructed from the cameras in $S$.

\begin{example}
Given parameters $m = 2$ and $d = 1$ in Figure~\ref{fig:cluster-example}, we take the cluster sequence $S = [CL_1, CL_2, CL_3]$ as an example. We have $CL_1.O \cap CL_2.O \cap CL_3.O = \{o_1, o_2, o_3\}$. Meanwhile, the conditions $0 < p^3_1 - p^2_1 < 2$ and $0 < p^2_1 - p^1_1 < 2$ hold. Therefore, we can conclude that $o_1$ is a core object of $S$. Similarly, $\lambda(S) = \{o_1, o_2, o_3\}$.
\end{example}


\begin{definition}\label{def:feasible-cls} \textbf{Feasible Cluster} \\
We call cluster $CL_i$ a feasible cluster for the cluster sequence $S = [CL_1,\dots, CL_{n}]$ if there exist at least $m$ objects $o_c \in CL_i.O$ such that $o_c \in \lambda(S)$ and $0 < p_c^i - p_c^{n} \leq d + 1$.
\end{definition}
Besides, we also define the cluster sequence $S = [CL_1, ..., CL_n]$ as a \blue{\textbf{feasible sequence}} if the last cluster $CL_n$ in $S$ is a feasible cluster for the prefix subsequence $[CL_1, ..., CL_{n-1}]$.

\begin{example}
We use Figure~\ref{fig:cluster-example} for illustration with parameters $m = 2$ and $d = 1$. The cluster $CL_3$ is a feasible cluster for $[CL_1, CL_2]$. This is because $CL_3.O \cap \lambda(S) = \{o_1, o_2, o_3\}$. Furthermore, we have $p^3_1 - p^2_1 = 1$, $p^3_2 - p^2_2 = 1$, and $p^3_3 - p^2_3 = 2$. Thus, there exist three objects (greater than $m$) that satisfy the conditions specified in Definition~\ref{def:feasible-cls}.
As a result, we conclude that $[CL_1, CL_2, CL_3]$ is a feasible sequence.
In contrast, $CL_4$ is not a feasible cluster for $[CL_1, CL_2]$ because $CL_1.O \cap CL_2.O \cap CL_4.O$ are $\{o3\}$ which consists of only one object.
\end{example}

The following lemma demonstrates the equivalence between the feasible sequence and the candidate co-movement pattern.

\begin{lemma}\label{lm:feasible-cls}
For a feasible sequence $S = [CL_1, \ldots, CL_n]$, the corresponding $\langle \lambda(S), CL_1.c \to \ldots \to CL_n.c \rangle$ must be a candidate co-movement pattern of length $n$ that satisfies $m$, $\epsilon$, and $d$.
\end{lemma}
\begin{proof}
We first consider the parameter $m$. Since $S$ is a feasible sequence, $CL_n$ must be a feasible cluster for $S_{n-1} = [CL_1, \dots, CL_{n-1}]$. In other words, there are at least $m$ objects $o_c$ in $CL_n \cap \lambda(S_{n-1})$ such that $0 < p_c^n - p_c^{n-1} \leq d + 1$. As $\lambda(S_{n-1})$ contains all the core objects of $S_{n-1}$, we further have $CL_n \cap \lambda(S_{n-1}) = \bigcap_{i=1}^{n} CL_i.O$ and $0 < p_c^i - p_c^{i-1} \leq d + 1 (1 < i \leq n)$. This implies that all the aforementioned at least $m$ objects are core objects of $S$.
As for the parameter $\epsilon$, the core objects of $\lambda(S)$ remain within the same cluster, which ensures that they are $\epsilon$-close at each camera.
Finally, we analyze the parameter $d$. For each core object $o_c \in \lambda(S)$, we can define an injection function $f: \{1, 2, \ldots, n\} \to \{p_c^1, p_c^2, \ldots, p_c^n\}$ from the common route $CL_1.c \to \ldots \to CL_n.c$ to $o_c$'s travel path $P_c$, such that the common route is a $d$-subpath of $P_c$.
The proof is complete.
\end{proof}

We will henceforth use \blue{the notation $\boldsymbol{R(S)}$} to denote the corresponding candidate co-movement pattern $\langle \lambda(S), CL_1.c \to \ldots \to CL_n.c \rangle$ for a feasible sequence $S = [CL_1, \ldots, CL_n]$. For example, in Figure~\ref{fig:cluster-example} with parameters $m = 2$ and $d = 1$, For the feasible sequence $S = [CL_1, CL_2, CL_3]$, $R(S)$ is $\langle \{o_1, o_2, o_3\}, A \to B \to D \rangle$.

\begin{figure}[t!]
	\centering
		\includegraphics[width=0.35\textwidth]{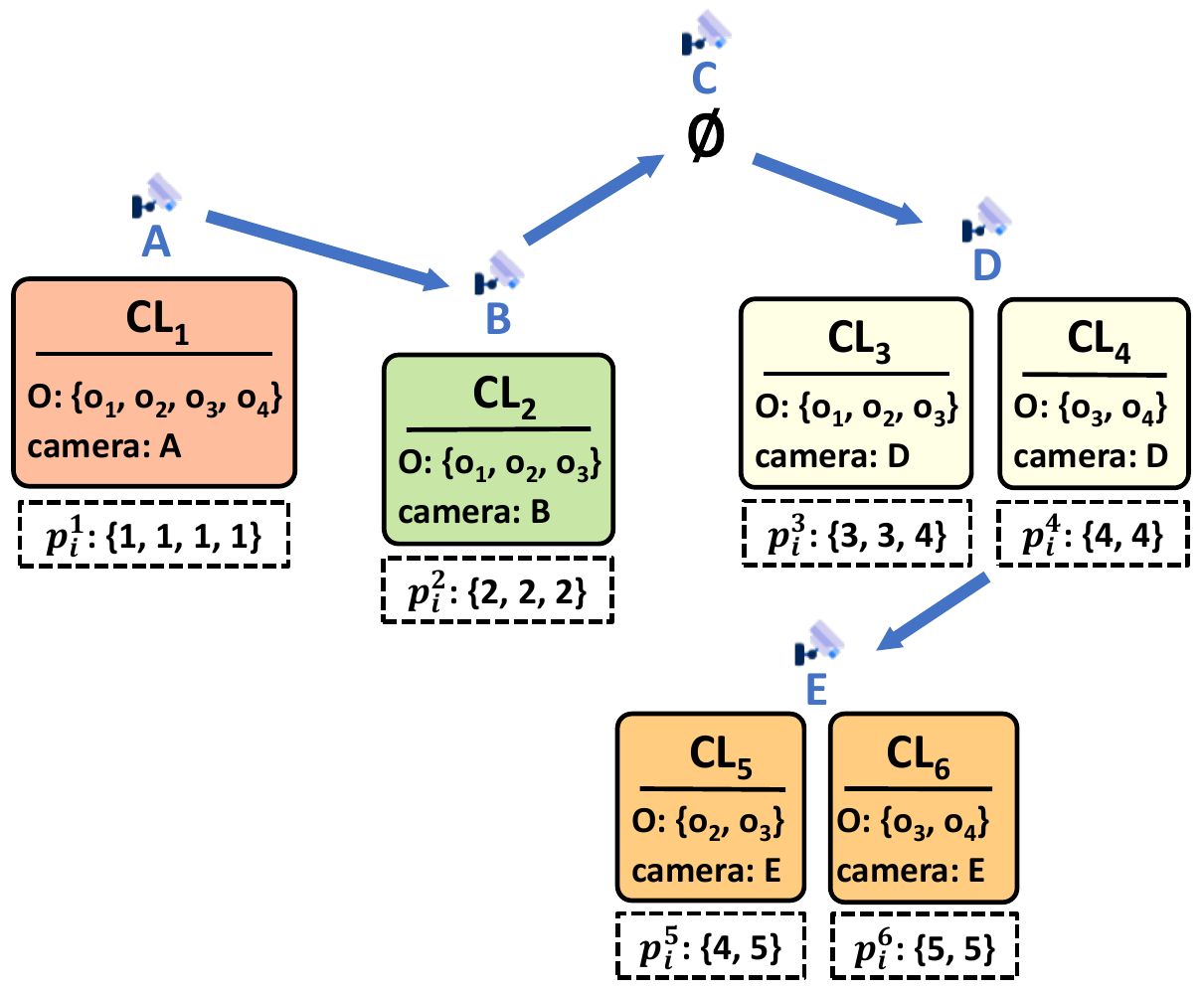}
	\caption{Example clusters under parameters $m = 2$ and $\epsilon = 6$.}
	\label{fig:cluster-example}
        \vspace{-4mm}
\end{figure}

\subsection{The Candidate Enumeration Scheme}\label{subsec:maxgrowth-scheme}
\blue{Based on the aforementioned cluster representation, the candidate enumeration scheme of MaxGrowth enumerates feasible sequences in a depth-first search manner for candidate generation.} Specifically, it starts with an empty sequence and selects feasible clusters for the current cluster sequence to grow a longer feasible sequence at each search step. Meanwhile, candidate co-movement patterns are retrieved from the enumerated feasible sequences.

\blue{Algorithm~\ref{alg:mgrowth} presents the complete pseudocode of MaxGrowth.} We first initialize $\mathbb{R}$ as the final result set (line 1). Next, the candidate enumeration scheme of MaxGrowth is executed. Specifically, all the clusters are computed from the input travel paths according to Definition~\ref{def:cls} for enumeration preparation (line 2). Each cluster is then used as the first element of the current cluster sequence $S$ to initiate the Growth routine for subsequent steps in candidate enumeration (lines 3-5). Finally, we obtain the final result set $\mathbb{R}$ after removing non-maximal patterns (lines 6-7). This step uses a technique similar to that in \cite{Zhang2023pvldb}.
We now continue to provide a detailed description of the Growth routine (lines 8-23), which is the core component of MaxGrowth's candidate enumeration scheme. 
In this routine, the feasible clusters for the current cluster sequence $S$ are first computed in lines 9-16. \blue{Specifically, we iterate over each core object $o_c$ of $S$, identifying candidate clusters that contain $o_c$ and satisfy the gap tolerance condition $0 < p_c^i - p_c^n \leq d + 1$ (lines 10-11).} $o_c$ must be one of the objects that satisfy the conditions specified in Definition~\ref{def:feasible-cls} within these candidate clusters. Therefore, each pair of candidate cluster and $o_c$ is recorded in the auxiliary dictionary $\mathrm{DT}$ (lines 12-14). If a candidate cluster contains at least $m$ core objects recorded in $\mathrm{DT}$, it is considered a feasible cluster and added to $\mathbb{FCL}$ (lines 15-16). 
After obtaining all feasible clusters, each of them is selected and appended to $S$ to continue the sequence growth routine (lines 17-20).
As the final step of the Growth routine, we check if the length of the current cluster sequence is at least $k$. If so, the corresponding candidate co-movement pattern $R(S)$ is directly added to the result set $\mathbb{R}$ as a valid pattern.

\begin{algorithm}[t!]
\SetAlgoNoEnd \SetAlgoNoLine 
\caption{\blue{MaxGrowth}}\label{alg:mgrowth}
\SetKwFunction{DFS}{Growth}
\SetKwProg{Fn}{Routine}{:}{}

$\mathbb{R} \leftarrow \emptyset$\;
$\mathbb{CL} \leftarrow$ compute all clusters using parameters $m$ and $\epsilon$\;
\ForEach{cluster $CL \in \mathbb{CL}$}{
    $S \leftarrow [CL]$\;
    \DFS{$S, \mathbb{R}$};
}
$\mathbb{R} \leftarrow$ remove non-maximal patterns in $\mathbb{R}$\;
\textbf{return} $\mathbb{R}$\;
\BlankLine

\Fn{\DFS{$S = [CL_1, \dots, CL_n], \mathbb{R}$}}{
    $\mathbb{FCL} \leftarrow \emptyset$; $\mathrm{DT} \leftarrow$ empty dictionary\;
    \ForEach{core object $o_c \in \lambda(S)$}{
	$\mathbb{CL}_c \leftarrow \{CL_i \mid o_c \in CL_i.O \land 0 < p^i_c - p^n_c \leq d + 1 \}$\;
	\ForEach{candidate cluster $CL_i \in \mathbb{CL}_c$}{
	   \lIf{$CL_i \notin \mathrm{DT}$}{$\mathrm{DT}[CL_i] \leftarrow \emptyset$}
	   $\mathrm{DT}[CL_i] \leftarrow \mathrm{DT}[CL_i] \cup \{ o_c \}$\;
	}
    }
    \ForEach{candidate cluster $CL_i \in \mathrm{DT}$}{
        \lIf{$|\mathrm{DT}[CL_i]| \geq m$}{$\mathbb{FCL} \leftarrow \mathbb{FCL} \cup \{ CL_i \}$}
    }
    \ForEach{feasible cluster $FCL \in \mathbb{FCL}$}{
        Append $FCL$ to $S$\;
        \DFS{$S, \mathbb{R}$}\;
        Discard $FCL$ from $S$\;
    }
    \lIf{$|S| \geq k$}{
        $\mathbb{R} \leftarrow \mathbb{R} \cup R(S)$
    }
}
\end{algorithm}

From the above description, we can see that MaxGrowth derives valid co-movement patterns directly from the execution of the candidate enumeration scheme, thus eliminating the need for the subsequent expensive verification procedure. 
Moreover, during candidate enumeration, MaxGrowth searches only within the feasible clusters relevant to the current cluster sequence. This well-restricted search space typically shrinks as the sequence grows, further ensuring the efficiency of pattern growth.

We now justify the correctness and completeness of MaxGrowth in detail. First, the following lemma demonstrates its correctness.

\begin{lemma}\label{lm:mg-correct}
MaxGrowth always generates valid co-movement patterns satisfying parameters $m$, $k$, $d$ and $\epsilon$ from candidate enumeration.
\end{lemma}
\begin{proof}
\blue{Since MaxGrowth only uses feasible clusters to grow the cluster sequence, each enumerated sequence $S$ must be a feasible sequence.} According to Lemma~\ref{lm:feasible-cls}, $S$ corresponds to a candidate co-movement pattern $R(S)$ that satisfies the parameters $m$, $d$, and $\epsilon$. Meanwhile, as the length of the common route in $R(S)$ is the same as that of $S$, and MaxGrowth generates candidate patterns only for corresponding sequences with a length of at least $k$, $R(S)$ must also satisfy the parameter $k$.
\end{proof}

The following lemma implies the completeness of MaxGrowth.

\begin{lemma}\label{lm:cgr-complete}
A valid relaxed co-movement pattern corresponds to at least one cluster sequence enumerated by MaxGrowth.
\end{lemma}
\begin{proof}
Let $R = \langle O, P \rangle$ be a relaxed co-movement pattern. Since $O$ satisfies $\epsilon$-close at each camera of $P$, there exists at least one cluster sequence $S = [CL_1, CL_2, \ldots, CL_n]$ such that $n = |P|$ and for $1 \leq i \leq n$, $O \subseteq CL_i.O$ and $CL_i.c$ is the $i$-th camera of $ P$. Since MaxGrowth will enumerate all feasible sequences, we now only need to prove that $S$ is a feasible sequence. First, we know $O \subseteq \bigcap_{i=1}^{n} CL_i.O$ contains at least $m$ objects. Meanwhile, since $P$ is the $d$-subpath of the travel path for each object in $O$, then $\forall o_j \in O$, $0 < p_j^i - p_j^{i-1} \leq d+1$ holds for $1 < i \leq n$. This further implies that $CL_n$ is a feasible cluster for the prefix subsequence $[CL_1, \dots CL_{n-1}]$ of $S$. This completes the proof.
\end{proof}

We can now derive the following theorem which guarantees the correctness and completeness of MaxGrowth.
\begin{theorem}
The MaxGrowth framework can generate both valid and complete relaxed co-movement patterns.
\end{theorem}
\begin{proof}
\blue{The proof is straightforward following from Lemma~\ref{lm:mg-correct} and Lemma~\ref{lm:cgr-complete}.}
\end{proof}

\section{Maximal Pattern Mining}\label{sec:max-pattern}
Building upon the candidate enumeration scheme of MaxGrowth, \blue{we introduce two efficient pruning rules that eliminate the generation of most maximal patterns during candidate enumeration.}

\begin{figure}[t!]
	\centering
		\includegraphics[width=0.48\textwidth]{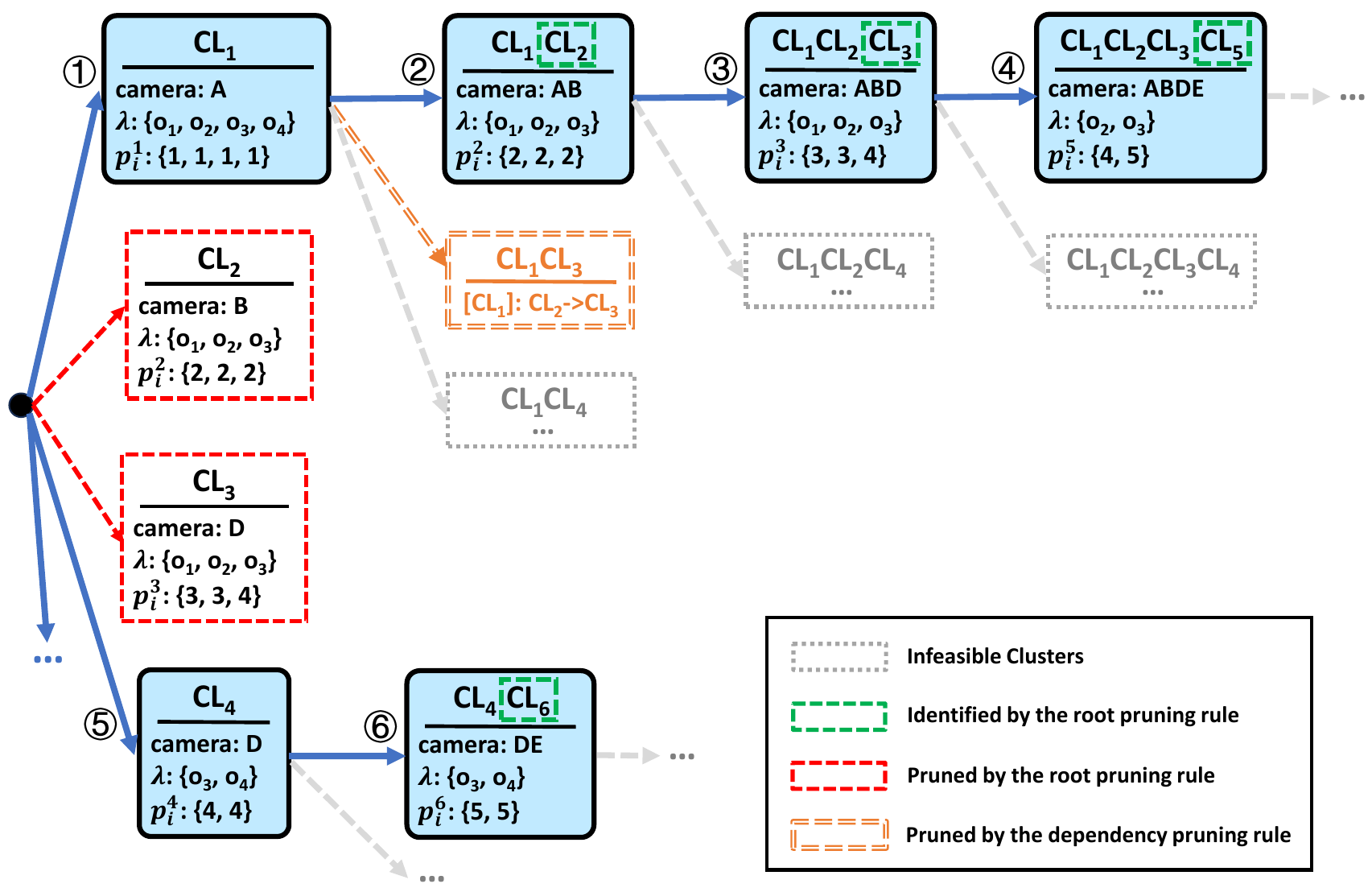}
	\caption{The search tree of MaxGrowth under parameters $m = 2$, $k = 2$, $d = 1$, and $\epsilon$ = 6.}
	\label{fig:search-tree}
\end{figure}

\subsection{The Root Pruning Rule}
\blue{The candidate enumeration scheme of MaxGrowth can be represented as a search tree, where each node corresponds to a selected feasible cluster and each branch represents a different direction of sequence growth.} Typically, there are a large number of clusters, resulting in an overwhelming number of subtrees at the root of the whole search tree. However, a significant number of these root subtrees cannot produce any maximal pattern. As shown in Figure~\ref{fig:search-tree}, only the search subtrees starting from clusters $CL_1$ and $CL_4$ can generate the maximal patterns $\langle \{o_2, o_3\}: A \rightarrow B \rightarrow D \rightarrow E \rangle$ and $\langle \{o_3, o_4\}: D \rightarrow E \rangle$, respectively.

To effectively prune those non-contributing subtrees for maximal pattern mining at the root of the search tree, we present the root pruning rule. The lemma below illustrates the core idea of it.

\begin{lemma}\label{lm:root-prune}
For a cluster $CL_i$, if there is a feasible sequence $S$ that $CL_i$ is a feasible cluster for $S$ and $CL_i.O \subseteq \lambda(S)$, then cluster sequences starting with $CL_i$ cannot correspond to maximal patterns.
\end{lemma}
\begin{proof}
Given a feasible sequence $S_n = [CL_1, CL_2, \ldots, CL_n]$ with $CL_1 = CL_i$, we prove $R(S_n)$ is not a maximal pattern.
For $1 < j \leq n$, cluster $CL_{j}$ is a feasible cluster for $S_{j-1} = [CL_1, \ldots, CL_{j-1}]$. Let $\tilde{S}_{j-1}$ denote the cluster sequence formed by concatenating $S$ before $S_{j-1}$. Then $CL_{j}$ is also a feasible cluster for $\tilde{S}_{j-1}$. This is because $CL_1.O = CL_i \subseteq \lambda(S)$ ensures that the core objects of $\tilde{S}_{j-1}$ are exactly the same as those of $S_{j-1}$.
Thus, let $\tilde{S}_n$ denote the sequence formed by concatenating $S$ before $S_n$, $\tilde{S}_n$ is also a feasible sequence with $\lambda(\tilde{S}_n) = \lambda(S_{n})$. This means that $\tilde{S}_n$ corresponds to a candidate pattern with the same objects as $S_{n}$ but a longer common path, indicating that $R(S_n)$ is not a maximal pattern.
\end{proof}

In Figure~\ref{fig:search-tree}, since the cluster $CL_2$ is a feasible cluster for $S = [CL_1]$ and $CL_2.O \subseteq$ \blue{$(CL_1.O = \lambda(S))$}, then any feasible sequence starting with $CL_2$ cannot correspond to a maximal pattern. As an example, the candidate pattern $\langle \{o_1, o_2, o_3\}, B \rightarrow D \rangle$ of the sequence $[CL_2, CL_3]$ is not maximal because there exists another maximal pattern $\langle \{o_1, o_2, o_3\}, A \rightarrow B \rightarrow D \rangle$.

Based on Lemma~\ref{lm:root-prune}, we state the following root pruning rule.
\begin{prule}\label{rule:root} \textbf{(The Root Pruning Rule)}
During the growth of a feasible sequence $S$, if a feasible cluster $CL_i$ satisfies $CL_i.O \subseteq \lambda(S)$, then the search subtree at the root starting with $CL_i$ can be pruned.
\end{prule}


It's worth noting that the effectiveness of the root pruning rule is influenced by the search order of clusters in MaxGrowth.
For example, in Figure~\ref{fig:search-tree}, if we search for cluster $CL_2$ before $CL_1$, the substree starting with $CL_2$ cannot be pruned by Rule~\ref{rule:root}. Because this subtree can only be identified as prunable by Rule~\ref{rule:root} during the growth of cluster sequences starting with $CL_1$. However, it has already been searched by that time.
To mitigate this issue, we define the binary relation "precedence" on all the clusters and utilize it to specify the search order of clusters in MaxGrowth.

\begin{definition}\label{def:precede} \textbf{Precedence Relation} \\
The precedence relation between two clusters $CL_u$ and $CL_v$ holds, denoted as $CL_u \leq_p CL_v$, if there is a series of clusters $CL_1, \ldots, CL_n$ such that:
\begin{itemize}
    \item $CL_1 = CL_u$ and $CL_n = CL_v$.
    \item For $1 \leq i < n$, $CL_i$ and $CL_{i+1}$ share at least one common object $o_j$ with $p_j^i \leq p_j^{i+1}$.
\end{itemize}
\end{definition}

\begin{example}
In the clusters shown in Figure~\ref{fig:cluster-example}, we have $CL_1 \leq_p CL_3$. This is because there exists a cluster series $CL_1, CL_2, CL_3$. Meanwhile, for $CL_1$ and $CL_2$, we can find a common object $o_1$ with $o_1 \in CL_1.O \land o_1 \in CL_2.O \land  p^1_1 < p^2_1$. For $CL_2$ and $CL_3$, we can also a common object $o_1$ with $o_1 \in CL_2.O \land o_1 \in CL_3.O \land p^2_1 < p^3_1$.
\end{example}

The precedence relation indicates the non-strict order in which objects move between clusters. In Definition~\ref{def:precede}, two adjacent clusters $CL_i$ and $CL_{i+1}$ share a common object $o_j$ with $p^i_j \leq p^{i+1}_j$ ensuring that at least one object will pass through camera $CL_i.c$ before camera $CL_{i+1}.c$. And the existence of cluster series from $CL_u$ to $CL_v$ allows this movement order to be transitive from $CL_u$ to $CL_v$.

Based on this relation, we determine the search order of clusters in MaxGrowth as follows: For clusters $CL_u$ and $CL_v$, if $CL_u \leq_p CL_v$ and $CL_v \not\leq_p CL_u$, then MaxGrowth will search $CL_u$ before $CL_v$.

Note that for two clusters $CL_u$ and $CL_v$ where both $CL_u \leq_p CL_v$ and $CL_v \leq_p CL_u$ are satisfied, we do not explicitly specify the search order between them. In practice, although symmetric pairs of precedence relations exist, their number is very small with the proportion less than ten percent. This is due to the inherent characteristics of clusters: Since objects within the same cluster exhibit temporal proximity, satisfying both $CL_u \leq_p CL_v$ and $CL_v \leq_p CL_u$ indicates that these two clusters simultaneously contain two sets of objects moving in opposite directions, which is rare and does not consistently occur across multiple clusters.

For clusters in Figure~\ref{fig:cluster-example}, since $CL_1 \leq_p CL_2$ and $CL_1 \not\leq_p CL_2$, MaxGrowth will search $CL_1$ before $CL_2$. In contrast, MaxGrowth will search $CL_3$ and $CL_4$ in arbitrary order. Thus, the overall search order is $CL_1 \to CL_2 \to \{ CL_3, CL_4 \} \to \{ CL_5, CL_6 \}$. With this order, Rule~\ref{rule:root} successfully identifies all target clusters including $CL_2$, $CL_3$, $CL_5$, and $CL_6$, and prunes their search subtrees at the root node.

Since symmetric pairs are rare in precedence relation, this allows us to determine the search order for most of clusters. The lemma below theoretically guarantees the effectiveness of this search order.

\begin{lemma}
If the precedence relation can specify the search order for all clusters, then the root pruning rule can prune the maximum number of search subtrees.
\end{lemma}
\begin{proof}
We proof by contradiction. Assuming that Rule~\ref{rule:root} cannot prune the maximum number of subtrees. There exists at least one cluster $CL_r$ that satisfies the pruning condition of Rule~\ref{rule:root} but is not recognized. Then $CL_r$ must be a feasible cluster for some candidate sequence $S = [CL_1,\ldots, CL_n]$ with $CL_r.O \subseteq \lambda_S$. Then objects in $CL_r.O$ pass through each camera in the cluster series $CL_1, \dots, CL_n, CL_r$ sequentially. Therefore, we have $CL_1 \leq_p CL_r$. Because the precedence relation specifies the order of all clusters, MaxGrowth will search $CL_1$ before $CL_r$ and identify the subtree starting with $CL_r$ as prunable by Rule~\ref{rule:root} during the growth of the candidate sequence $S$. This contradicts the assumption.
\end{proof}

\subsection{The Dependency Pruning Rule}
Although the root pruning rule can prune subtrees at the root of the search tree, it cannot prune subtrees at higher levels during candidate enumeration. In fact, most maximal patterns are generated from the search subtrees of non-root nodes. Specifically, in the basic candidate scheme of MaxGrowth, we need to append each feasible cluster to the end of the current cluster sequence for further growth. However, searching through all feasible clusters for a cluster sequence as described above is unnecessary and time-consuming, as many subsequent search subtrees of feasible clusters will not produce any maximal patterns.

\blue{To facilitate more refined optimization of the search space for non-maximal patterns, we introduce the dependency pruning rule.} The main idea of this rule is to determine non-essential search feasible clusters by analyzing whether a specific relationship between feasible clusters is satisfied. We call this specific relationship as \textbf{the dependency relationship} and describe it in detail below.

For a cluster sequence $S$ and its feasible cluster $CL_i$, let the sequence $S_i = S \;||\; [CL_i]$ represents the concatenation of two sequences $S$ and $[CL_i]$. Then, the dependency relationship between two feasible clusters is defined as follow:

A feasible cluster $CL_i$ depends on $CL_j$ for a same feasible sequence $S$, if object $o_c \in \lambda(S_i) \rightarrow (o_c \in \lambda(S_j) \land p_c^j < p_c^i)$.

Intuitively, if a feasible cluster $CL_i$ depends on another feasible cluster $CL_j$, it actually means that any object in $\lambda(S_i)$ must pass through camera $CL_j.c$ before reaching camera $CL_i.c$. In other words, $S_i$ not only lacks information about the common route of objects passing through camera $CL_j.c$, but its core objects can also be obtained through the further growth of $S_j$.

\begin{example}
Given $m = 2$, $d = 1$, and $\epsilon = 6$ in Figure~\ref{fig:cluster-example}, there are two feasible clusters $CL_2$ and $CL_3$ for the feasible sequence $S = [CL_1]$.
For these two clusters, we can also derive that $S_2 = [CL_1, CL_2]$, $S_3 = [CL_1, CL_3]$ and $\lambda(S_2) = \lambda(S_3) = \{o_1, o_2\}$. Since object $o_1$ satisfies $o_1 \in \lambda(S_2)$ as well as $o_1 \in \lambda(S_2) \land p_1^2 < p_1^3$; and object $o_2$ satisfies $o2 \in \lambda(S_3)$ as well as $o_2 \in \lambda(S_3) \land p_2^2 < p_2^3$ , we can conclude that $CL_3$ depends on $CL_2$.
\end{example}

The following lemma \blue{demonstrates} how to prune non-essential subsequent search subtrees for feasible clusters based on the aforementioned dependency relationship.

\begin{lemma}\label{lm:depend-prune}
If cluster $CL_i$ depends on $CL_j$ for a feasible sequence $S$, then the subsequent sequence growth after appending $CL_i$ to $S$ will not produce any maximal patterns.
\end{lemma}
\begin{proof}
Assuming $S = [CL_1, CL_2, ..., CL_n]$, let $S_i$ denote the cluster sequence $S \;||\; [CL_i]$ and $S_{ji}$ denote $S \;||\; [CL_j, CL_i]$, where $||$ still represents the concatenation of two cluster sequences.
We prove that the search tree rooted at $S_{i}$ and $S_{ji}$ are identical. Since the last cluster of $S_i$ and $S_{ji}$ are both $CL_i$, this is equivalent to proving their core objects are the same.
Based on the dependency relationship, each $o_c \in \lambda(S_i)$ satisfies the following conditions: (1) $p_c^i - p_c^n \leq d + 1$, (2) $p_c^i - p_c^j > 0$, (3) $p_c^j - p_c^n > 0$. From (3) and (1) - (2), we can obtain $0 < p_c^j - p_c^n < d + 1$.
Thus, $\lambda(S_i) \subseteq \lambda(S_j)$ and $\lambda(S_i) \subseteq (CL_i.O \cap \lambda(S_j))$. Meanwhile, from (2) and (1) - (3), we have $0 < p_c^i - p_c^j < d + 1$. Therefore, $S_{ji}$ is a feasible sequence with core objects that contain at least $\lambda(S_i)$. As a result, $S_i$ and $S_{ji}$ have the same core objects.
\end{proof}

Based on Lemma~\ref{lm:depend-prune}, we state the dependency pruning rule.

\begin{prule}\label{rule:depend} \textbf{(The Dependency Pruning Rule)}
During the growth of a feasible sequence $S = [CL_1, CL_2, ..., CL_n]$, if a feasible cluster $CL_i$ depends on at least one other feasible cluster, then the search substree for $CL_i$ after $S$ can be pruned.
\end{prule}

As shown in Figure~\ref{fig:search-tree}, we know from the previous example that for candidate sequence $S = [CL_1]$, the feasible cluster $CL_3$ depends on $CL_2$. Therefore, we do not need to append $CL_3$ to $S$ for further sequence growth.

The dependency pruning rule prunes on non-root search nodes during sequence growth (candidate enumeration) by detecting dependency relationships between feasible clusters. Specifically, if a feasible cluster $CL_i$ can be pruned, it means that searching $CL_i$ directly will lead to the same subsequent search subtree as first searching the clusters that $CL_i$ depends on and then searching $CL_i$. Therefore, there is no need to continue searching $CL_i$ after $S$.

\section{Experimental Evaluation} \label{sec:experiment}
In this section, we evaluate the effectiveness of the relaxed co-movement pattern and the scalability of MaxGrowth algorithm. 

\subsection{Experimental Setup}\label{subsec:exp-setup}

\textbf{Datasets.} Following previous work~\cite{Zhang2023pvldb}, we use real GPS trajectories and road network, including DIDI Chengdu~\cite{tong2018didi} and Singapore Taxi~\cite{fan2016spare}, to generate approximate trajectories recovered from videos. The idea is to deploy a specified number of cameras  on the road network. Then,  the entrance time and exit time from a camera according to the travel speed estimation and the attributes of the camera can be roughly estimated from the GPS trajectories. 
\begin{itemize}
    \item \textbf{Singapore}: This is a large-scale cross-camera trajectory dataset which contains travel routes of 2,756 taxis in the road network of Singapore over one month.
    \item \textbf{Chengdu}: This dataset contains cross-camera trajectories of 12,000 ride-hailing vehicles in Chengdu, China.
\end{itemize}
Besides these two semi-synthetic datasets from GPS trajectories, we also construct a real dataset from video data:
\begin{itemize}
    \item \textbf{CityFlow}~\cite{tang2019cityflow}: This video dataset contains $215$ minutes of videos collected from $46$ cameras spanning 16 intersections in a mid-sized U.S. city. We adopt the multi-camera multi-target tracking algorithm~\cite{yang2022box} to recover movement paths of objects in the CityFlow videos. The recovered trajectory dataset includes cross-camera trajectories of $337$ unique objects over a duration of approximately 5 minutes.
\end{itemize}

Since CityFlow is a small-scale dataset, we employ Singapore and Chengdu for scalability analysis. CityFlow serves to measure the effectiveness of relaxed co-movement pattern mining.

\blue{
\begin{table}[ht!]
    \color{blue}
    \captionsetup{labelfont={color=blue, bf},textfont={color=blue, bf}, font=normalsize}
    \small
    \centering
    \vspace{-4mm}
    \caption{Statistics of datasets.}
    \vspace{-2mm}
    \label{tbl:dataset}
    \begin{tabular}{|l|c|c|c|}
        \hline
        \textbf{Attributes} & \textbf{Singapore} & \textbf{Chengdu} & \textbf{CityFlow} \\
        \hline
        \# objects & $2\text{,}756$ & $12\text{,}000$ & $337$ \\
        \hline
        \# cameras & $37\text{,}370$ & $9\text{,}476$ & $19$ \\
        \hline
        \# data points & $2\text{,}458\text{,}742$ & $1\text{,}841\text{,}071$ & $2\text{,}563$ \\
        \hline
        Time span per vehicle & $6\text{,}246$s & $6\text{,}801$s & $59$s \\
        \hline
        \makecell[l]{Avg. trajectory length \\[-0.26em] (\# cameras per vehicle)} & $892$ & $153$ & $8$ \\
        \hline
    \end{tabular}
    \vspace{-1mm}
\end{table}
}
 
\noindent\textbf{Comparison Setup.} 
In the scalability analysis, we compare MaxGrowth with the baseline algorithm FRB, which is an extension of previous video-based co-movement pattern miner and presented in Section~\ref{sec:baseline}.
\blue{
For the effectiveness analysis, we compare our proposed relaxed co-movement pattern (referred to as VPlatoon) with previous co-movement pattern from \cite{Zhang2023pvldb} (referred to as VConvoy) in terms of their effectiveness for pattern discovery.
}

\noindent\textbf{Parameter Setup.} We examine the scalability performance w.r.t. the parameters in Table~\ref{tbl:params}, including four query-related parameters ($m$, $k$, $d$ and $\epsilon$) and \blue{three} database-related parameters (\blue{camera number}, object number and averag path length). \blue{The default parameters are highlighted in bold.}

All the experiments are conducted on a server with $3.20$GHz i9-12900K CPU, $256$GB of memory and $2$TB hard drive. 

\begin{table}[t!]
\centering
\caption{Evaluation parameter settings.}
\label{tbl:params}
\resizebox{\linewidth}{!}{
\begin{tabular}{|c|c|l|} \hline
\multirow{2}{*}{Group size $m$} & Singapore &  \multirow{2}{*}{\blue{\textbf{3}, 4, 5, 6, 7, 8}}\\ \cline{2-2}
& Chengdu &  \\ \hline

\multirow{2}{*}{Common cameras $k$} & Singapore &  \multirow{2}{*}{\blue{2, \textbf{3}, 4, 5, 6, 7, 8}}\\ \cline{2-2}
 & Chengdu & \\ \hline

\multirow{2}{*}{Gap tolerance $d$} & Singapore &  \multirow{2}{*}{0, 1, 2, \textbf{3}, 4}\\ \cline{2-2}
& Chengdu & \\ \hline

\multirow{2}{*}{Proximity $\mathbf{\epsilon}$} & Singapore & \multirow{2}{*}{\blue{$40$, $50$, $\mathbf{60}$, $70$, $80$, $90$, $100$}}  \\ \cline{2-2}
& Chengdu & \\ \hline

\multirow{2}{*}{\blue{Camera number}} & \blue{Singapore} & \blue{$5$k, $10$k, $15$k, $20$k, $25$k, $30$k, \textbf{35k}} \\ \cline{2-3}
& \blue{Chengdu} & \blue{$3$k, $4$k, $5$k, $6$k, $7$k, $8$k, \textbf{9k}} \\ \hline

\multirow{2}{*}{Object number} & Singapore & \blue{$1.5$k, $1.7$k, $1.9$k, $2.1$k, $2.3$k, $2.5$k, \textbf{2.7k}} \\ \cline{2-3}
& Chengdu & \blue{$0.5$k, \textbf{2.5k}, $4.5$k, $6.5$k, $8.5$k, $10.5$k, $12$k} \\ \hline

\multirow{2}{*}{Average path length} & Singapore & \blue{$200$, $300$, $400$, $500$, $600$, $700$, $\mathbf{800}$}  \\ \cline{2-3}
& Chengdu & \blue{$90$, $100$, $110$, $120$, $130$, $140$, $\mathbf{150}$} \\ \hline
\end{tabular}
}
\end{table}

\subsection{Scalability Analysis}

\textbf{Varying $\mathbf{m}$}. We first analyze the performance of the relaxed co-movement pattern mining algorithms w.r.t. group size $m$. As shown in Figure~\ref{exp:vary-m-time}, when $m$ is small, MaxGrowth demonstrates remarkably faster performance than the baseline algorithm FRB. Specifically, MaxGrowth’s runtime is reduced by two orders of magnitude when $m=3$ in Chengdu dataset. \blue{Furthermore, when $m = 2$, FRB fails to complete within $3\text{,}600$ seconds on both datasets, while MaxGrowth finishes mining in $18.41$ seconds on Singapore dataset and $2.6$ seconds on Chengdu dataset. We omit this result to avoid distortion.} The high efficiency of MaxGrowth stems from its ability to avoid the time-consuming verification of false positives and effectively discard non-maximal candidate patterns through pruning rules.
As $m$ increases, the performance gap narrows, with FRB showing similar efficiency to MaxGrowth, due to the significant reduction in the number of valid candidate patterns.

\begin{figure}[h!]
    \vspace{-3mm}
  \centering
  \subfigcapskip=-6pt
  \subfigbottomskip=-3pt
  \subfigure[Singapore]{
    \includegraphics[width=0.235\textwidth]{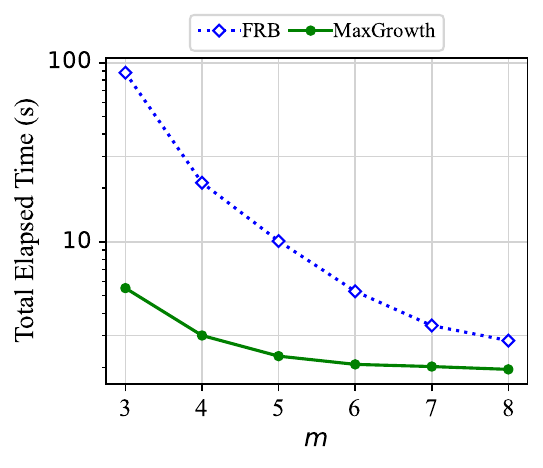}
  }
  \hspace{-4mm}
  \subfigure[Chengdu]{
    \includegraphics[width=0.235\textwidth]{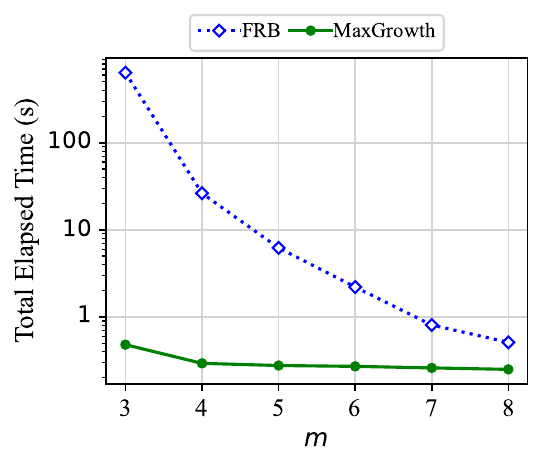}
  }
  \vspace{-3mm}
  \caption{Varying $m$.}
  \label{exp:vary-m-time}
 \vspace{-3mm}
\end{figure}

\noindent\textbf{Varying $\mathbf{\epsilon}$}. We examine the performance with increasing $\epsilon$ in Figure~\ref{exp:vary-eps-time}. MaxGrowth consistently outperforms FRB mainly because FRB requires heavy computation overhead to verify the constraint of temporal proximity. In contrast, the candidate enumeration scheme in MaxGrowth eliminates the verification cost.
When $\epsilon$ is large, the search space grows dramatically. The widened performance gap between MaxGrowth and its competitor again validates the effectiveness of its enumeration and pruning schemes.


\begin{figure}[h!]
  \vspace{-3mm}
  \centering
  \subfigcapskip=-6pt
  \subfigbottomskip=-3pt
  \subfigure[Singapore]{
    \includegraphics[width=0.235\textwidth]{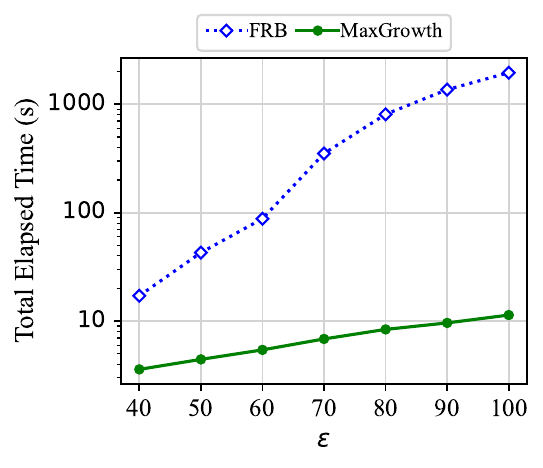}
  }
  \hspace{-4mm}
  \subfigure[Chengdu]{
    \includegraphics[width=0.235\textwidth]{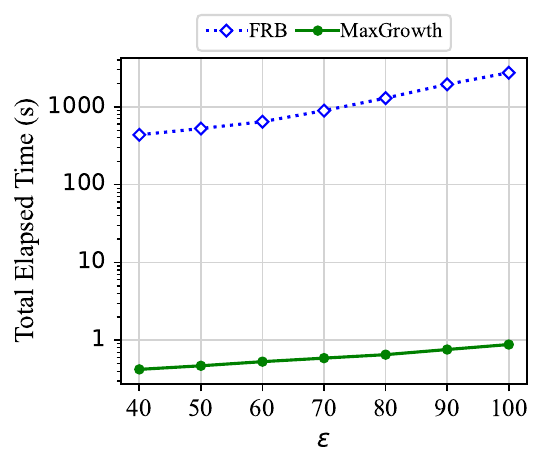}
  }
  \vspace{-3mm}
  \caption{Varying  $\mathbf{\epsilon}$.}
  \label{exp:vary-eps-time}
 \vspace{-3mm}
\end{figure}

\noindent\textbf{Varying $\mathbf{k}$}. 
As depicted in Figure~\ref{exp:vary-k-time}, the MaxGrowth algorithm is not sensitive to the variation of $k$. This is because its candidate patterns are progressively enumerated with  increasing route length. In other words, regardless of $k$,  all candidate patterns that meet the group size and temporal proximity constraints are enumerated before undergoing the maximal pattern mining. The slight decline observed in the MaxGrowth algorithm is attributed to the reduction in the number of valid candidate patterns as $k$ increases, leading to a lower computational cost for identifying maximal patterns. In contrast, FRB employs sequence mining to filter candidates with route lengths shorter than $k$, allowing more candidates to be excluded as $k$ grows, thus reducing the overall mining overhead.


\begin{figure}[h!]
\vspace{-3mm}
  \centering
  \subfigcapskip=-6pt
  \subfigbottomskip=-3pt
  \subfigure[Singapore]{
    \includegraphics[width=0.235\textwidth]{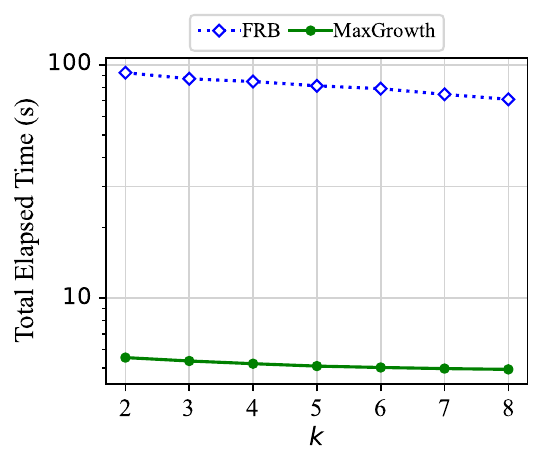}
  }
  \hspace{-4mm}
  \subfigure[Chengdu]{
    \includegraphics[width=0.235\textwidth]{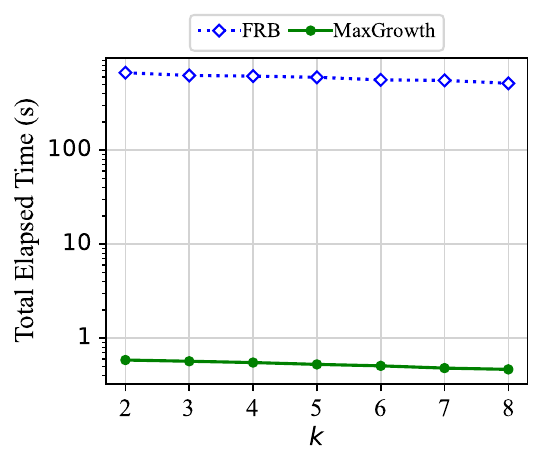}
  }
  \vspace{-3mm}
  \caption{Varying $\mathbf{k}$.}
  \label{exp:vary-k-time}
 \vspace{-3mm}
\end{figure}

\noindent\textbf{Varying $\mathbf{d}$}. Recall that we introduce an additional parameter $d$ in this work to allow missing cameras for relaxed co-movement pattern mining. When $d=0$, the problem is reduced to traditional co-movement pattern mining. FRB demonstrates comparable performance with MaxGrowth because FRB is customized for the setting with consecutive camera sequences. When the condition is relaxed, its running time increases remarkably, paying large amount of computation overhead for candidate verification. MaxGrowth achieves the best performance and its running time remains stable with increasing $d$.  When $d=4$, the time cost of MaxGrowth is two orders of magnitude lower than FRB.

\begin{figure}[t!]
\vspace{-3mm}
  \centering
  \subfigcapskip=-6pt
  \subfigbottomskip=-3pt
  \subfigure[Singapore]{
    \includegraphics[width=0.235\textwidth]{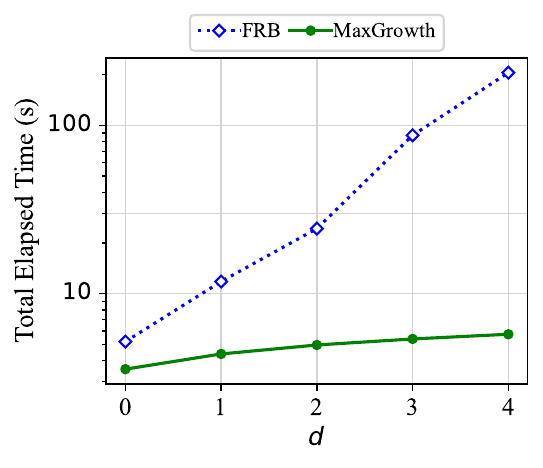}
  }
  \hspace{-4mm}
  \subfigure[Chengdu]{
    \includegraphics[width=0.235\textwidth]{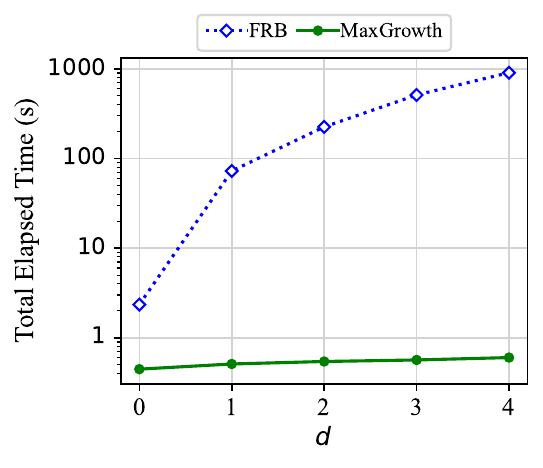}
  }
  \vspace{-3mm}
  \caption{Varying $\mathbf{d}$.}
  \label{exp:vary-d-time}
 \vspace{-3mm}
\end{figure}

\noindent\textbf{Varying Object Number}. 
\blue{We continue to analyze the performance with respect to database-related parameters. The results for increasing dataset cardinality are presented in Figure~\ref{exp:vary-object-time}.}
The outcomes share similar patterns with previous experiments --- the running time increases with higher number of objects and the performance gap is widened with larger-scale dataset. In Chengdu, the running time of FRB is two orders of magnitude higher than MaxGrowth.

\begin{figure}[h!]
\vspace{-3mm}
  \centering
  \subfigcapskip=-6pt
  \subfigbottomskip=-3pt
  \subfigure[Singapore]{
    \includegraphics[width=0.235\textwidth]{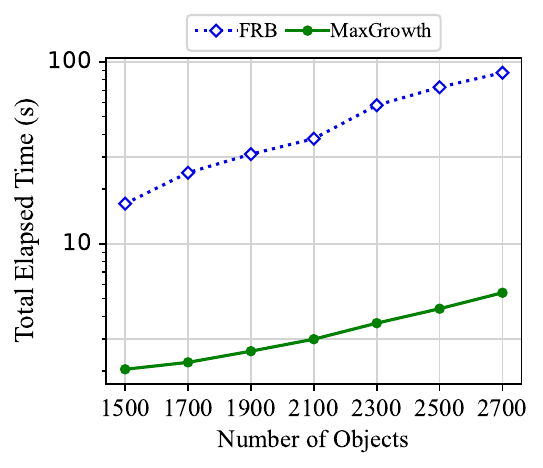}
  }
  \hspace{-4mm}
  \subfigure[Chengdu]{
    \includegraphics[width=0.235\textwidth]{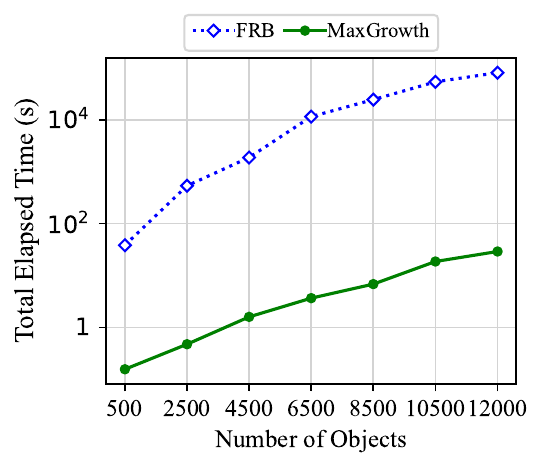}
  }
  \vspace{-3mm}
  \caption{Varying object number.}
  \label{exp:vary-object-time}
 \vspace{-3mm}
\end{figure}

\noindent\textbf{Varying Path Length}.
\blue{In Figure~\ref{exp:vary-path-time}, we examine the performance with increasing travel path length for the moving vehicles.}
We construct subsets of Singapore and Chengdu with increasing path length. MaxGrowth is also scalable to this parameter and outperforms FRB with a large margin.

\begin{figure}[h!]
\vspace{-3mm}
  \centering
  \subfigcapskip=-6pt
  \subfigbottomskip=-3pt
  \subfigure[Singapore]{
    \includegraphics[width=0.235\textwidth]{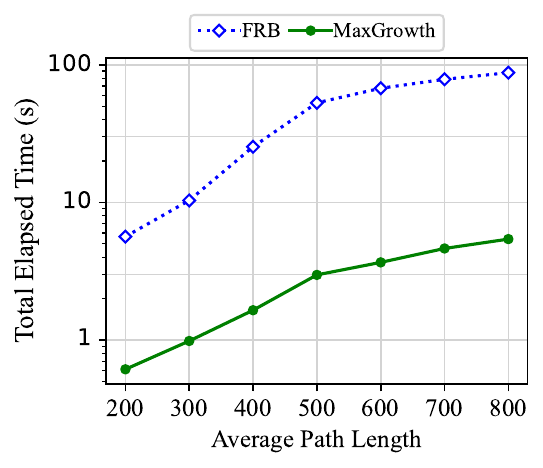}
  }
  \hspace{-4mm}
  \subfigure[Chengdu]{
    \includegraphics[width=0.235\textwidth]{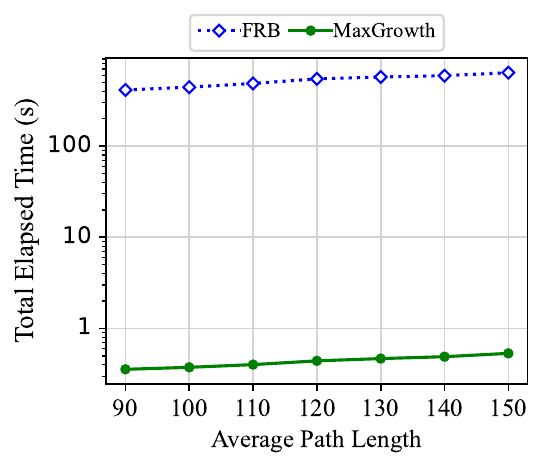}
  }
  \vspace{-3mm}
  \caption{Varying path length.}
  \label{exp:vary-path-time}
 \vspace{-3mm}
\end{figure}

\noindent \blue{\textbf{Varying Camera Number}}. 
\blue{In the final scalability experiment, we report the mining performance w.r.t. the number of cameras. As shown in Figure~\ref{exp:vary-camera-time}, MaxGrowth again outperforms FRB across all camera settings, further demonstrating its great efficiency.
Another noteworthy observation is that both FRB and MaxGrowth exhibit low sensitivity to the number of cameras compared to other database-related parameters. This is because deploying more cameras in the same road network does not result in a significant increase in the number of co-movement pattern candidates when the object set and the average trajectory length remain unchanged.
}

\begin{figure}[h!]
  \captionsetup{labelfont={color=blue, bf},textfont={color=blue, bf}, font=normalsize}
  \centering
  \subfigure[Singapore]{
    \includegraphics[width=0.235\textwidth]{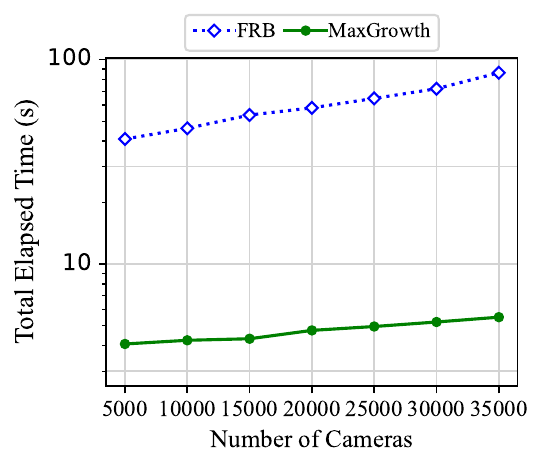}
  }
  \hspace{-4mm}
  \subfigure[Chengdu]{
    \includegraphics[width=0.235\textwidth]{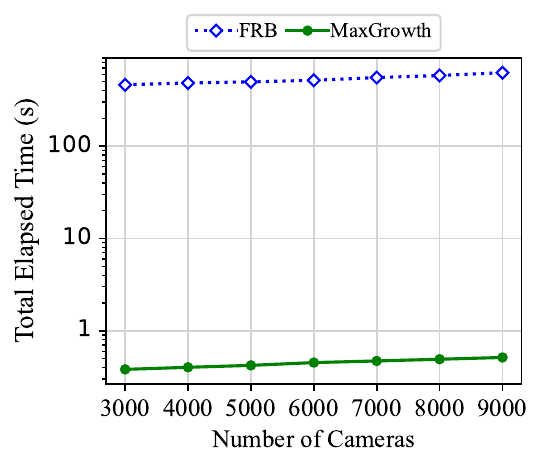}
  }
  \caption{\blue{Varying camera number.}}
  \label{exp:vary-camera-time}
\end{figure}

\subsection{In-Depth Analysis of MaxGrowth}
\textbf{Running Time Breakdown Analysis.} To further validate the effectiveness of the proposed techniques in MaxGrowth, we first perform a breakdown analysis on the running time of each component, including candidate generation, validness verification and dominance verification. Table~\ref{tab:breakdown} reports the time cost of each component in FRB and MaxGrowth. The cost of validness verification is the main performance bottleneck for FRB in both Singapore and Chengdu datasets. MaxGrowth can eliminate the requirement for candidate verification and thus save considerable amount of verification cost. Meanwhile, the candidate generation and dominance verification costs of FRB are also higher than those of MaxGrowth. This is because FRB cannot effectively avoid the search and generation of dominated non-maximal candidate patterns, leading to a larger search space during candidate generation and increased cost for dominance verification. In contrast, MaxGrowth significantly reduces these overheads through two pruning rules for maximal pattern mining.


\begin{table}[h!]
\centering
\vspace{-1mm}
\caption{Running time breakdown analysis.}
\vspace{-2mm}
\label{tab:breakdown}
\resizebox{0.9\linewidth}{!}{
\begin{tabular}{|c|c|c|c|c|}
\hline
\multirow{3}{*}{\textbf{Datasets}} & \multirow{3}{*}{\textbf{Methods}} & \multicolumn{3}{c|}{\textbf{Runtime of the breakdown stage (s)}} \\ \cline{3-5} 
 &  & \makecell{Candidate \\[-0.3em] Generation} & \makecell{Validness \\[-0.3em] Verification} & \makecell{Dominance \\[-0.3em] Verification} \\ \hline
\multirow{2}{*}{Singapore} & FRB & 20.2 & 43.8 & 12.1 \\ \cline{2-5} 
 & MaxGrowth & 8.5 & - & 3.0 \\ \hline
\multirow{2}{*}{Chengdu} & FRB & 92.3 & 444.5 & 96.6 \\ \cline{2-5} 
 & MaxGrowth & 5.3 & - & 2.3 \\ \hline
\end{tabular}
}
\vspace{-1mm}
\end{table}

\noindent\textbf{Ablation Study.} We conduct ablation experiments to further investigate the impact of pruning rules on MaxGrowth. We design three additional variants of MaxGrowth. Among these variants, MG-Root employs only the root pruning rule, MG-Dep uses only the dependency pruning rule, and MG-NoPrune does not utilize either pruning rule. As shown in Table~\ref{tbl:ablation}, both the root pruning rule and the dependency pruning rule effectively prune redundant search subtrees, resulting in the clear performance improvement. Moreover, because the two rules prune at different levels of the search tree, MaxGrowth can further enhance efficiency. In the Chengdu dataset that contains numerous redundant patterns, while the root pruning rule achieves a 50\% performance improvement through greatly pruning the search tree at the root node, many non-maximal patterns still emerge in subsequent search subtrees, which limits overall performance. In this scenario, the dependency pruning rule provides a decisive enhancement by enabling fine-grained pruning during pattern growth.


\begin{table}[h!]
\centering
\vspace{-1mm}
\caption{Effect of candidate pruning.}
\vspace{-2mm}
\label{tbl:ablation}
\resizebox{0.9\linewidth}{!}{
\begin{tabular}{|c|c|c|c|c|}
\hline
\multirow{2}{*}{\textbf{Datasets}} & \multicolumn{4}{c|}{\textbf{The number of non-maximal patterns}} \\ \cline{2-5} 
                          & MG-NoPrune & MG-Root & MG-Dep & MaxGrowth \\ \hline
Singapore                 & 2,097,851  & 1,402,269 & 59,329 & 31,067    \\ \hline
Chengdu                   & 15,961,467 & 8,568,710 & 5,424 & 2,075    \\ \hline
\end{tabular}
}
\vspace{-2mm}
\end{table}

\subsection{Effectiveness Analysis}
In this part, we perform quality analysis on the relaxed co-movement patterns discovered according to our problem definition. The experiment is conducted on the real video dataset CityFlow, which was designed for the task of multi-camera multi-object tracking. We use its annotated trajectories as groundtruth and derive golden co-movement patterns. $F_1$-score is used as the performance metric. \blue{A discovered co-movement pattern is considered a positive match with some golden pattern by default if the intersection over union ($IoU$) of their objects and time span are both greater than $80\%$.}

\begin{table}[h!]
\small
\centering
\vspace{-1mm}
\caption{Comparison of pattern discovery effectiveness.}
\vspace{-2mm}
\label{tbl:quality}
\begin{tabular}{|c|c|c|c|}
\hline
\textbf{Pattern Type} & $\mathbf{F_1}$\textbf{-score} & \textbf{Precision} & \textbf{Recall} \\
\hline
VConvoy & 0.719 & 0.861 & 0.617 \\
\hline
\blue{VConvoy-TMerge} & \blue{0.742} & \blue{0.874} & \blue{0.644} \\
\hline
VPlatoon & 0.840 & 0.826 & 0.855 \\
\hline
\end{tabular}
\vspace{-1mm}
\end{table}

\blue{To evaluate the effectiveness of VPlatoon and VConvoy, we perform these two algorithms on the trajectories recovered from CityFlow. Additionally, we construct a competitive  baseline that applies the idea of TMerge~\cite{chao2023tmerge} as a post-processing step to improve the accuracy of the recovered trajectories. Certain trajectory segments split by the missing observations can be merged according to spatial and temporal clues. The new baseline performs VConvoy against the refined dataset and we call it VConvoy-TMerge.}
As to the pattern parameters, we set the minimum number of objects $m$ to $2$, the minimum route length $k$ to $3$, the proximity threshold $\epsilon$ to $60$ seconds, and the distance parameter $d$ for VPlatoon to $3$. Since the input dataset is small-scale, \blue{it} takes 0.03 seconds to run the TCS-tree algorithm for VConvoy pattern mining and 0.05 seconds to run MaxGrowth for VPlatoon pattern mining. \blue{However, VConvoy-TMerge incurs a considerable cost of 5.4 hours to obtain the refined dataset with parameters $\tau = 800\text{,}000$ and $K = 0.001$.}

The $F_1$-score of discovered patterns are reported in Table~\ref{tbl:quality}. VPlatoon achieves significantly higher $F_1$-score, because VConvoy has missed many valid patterns and its recall is remarkably lower than VPlatoon. In terms of precision, though the pattern is relaxed, the precision of VPlatoon is still comparable to VConvoy, implying that the number of false positive patterns introduced by pattern relaxation is limited.
\blue{Moreover, there is a slight performance improvement for VConvoy-TMerge, yet its $F_1$-score still lags behind VPlatoon by nearly $10\%$. This implies that current trajectory accuracy improvement techniques alone may be insufficient for VConvoy to discover most of the missing patterns. In contrast, our proposed VPlatoon demonstrates greater effectiveness at a lower cost.}

\blue{In our default settings, the positive matching thresholds for performance evaluation are set at $80\%$. To justify this, we analyze the quality distribution characteristics of the discovered co-movement patterns. Specifically, for each discovered pattern, we compute its maximum Intersection over Union (IoU) with the ground truth patterns, where both object set IoU and time span IoU are considered in the same manner as our positive matching rule. The distribution of these maximum IoU across all discovered patterns are illustation in Figure~\ref{exp:f1-dst}. Both VPlatoon and VConvoy exhibit consistent IoU distributions, with the largest proportion of discovered patterns having an IoU around $0.8$. This suggests that when the IoU positive matching thresholds are set above $80\%$, most high-quality patterns cannot be correctly matched, whereas thresholds lower than $80\%$ increase the risk of mismatching low-quality and unrepresentative patterns as positives. Therefore, setting the default matching thresholds to $80\%$ would be a balanced and reasonable choice.}

\begin{figure}[h!]
    \vspace{-3mm}
    \setlength{\abovecaptionskip}{6pt}
    \subfigcapskip=-9pt
    \subfigbottomskip=-3pt
    
    \centering
    \includegraphics[width=0.27\textwidth]{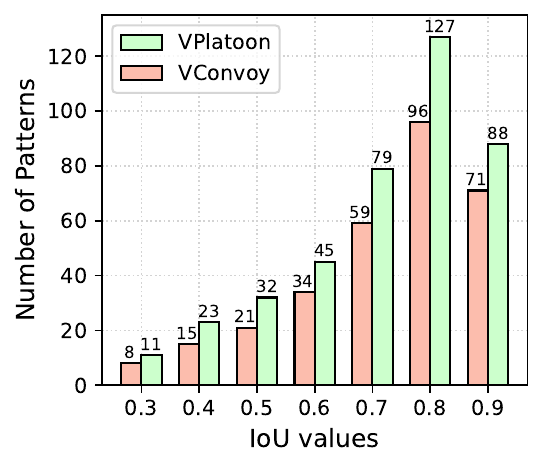}
    \vspace{-2mm}
    \caption{Distribution of IoU values for discovered patterns.}
    \label{exp:f1-dst}
    
    \vspace{-2mm}
\end{figure}

To better understand the relationship between our introduced parameter $d$ and the patterns discovered by VPlatoon, we present in Figure~\ref{exp:f1-d} the comparison of VPlatoon-mined patterns with the ground truth across different values of $d$. As $d$ increases, the precision gradually decreases while the recall increases, with the optimal $F_1$-score achieved at $d = 3$. This is because parameter $d$ controls the degree of relaxation by adjusting the distance between cameras in the common route. Thus, a larger $d$ enables the discovery of more patterns but also introduces potential irrelevant results which may reduce the precision.

\begin{figure}[h!]
    \vspace{-2mm}
    \setlength{\abovecaptionskip}{3pt}
    \subfigcapskip=-9pt
    \subfigbottomskip=-3pt
    
    \centering
    \includegraphics[width=0.28\textwidth]{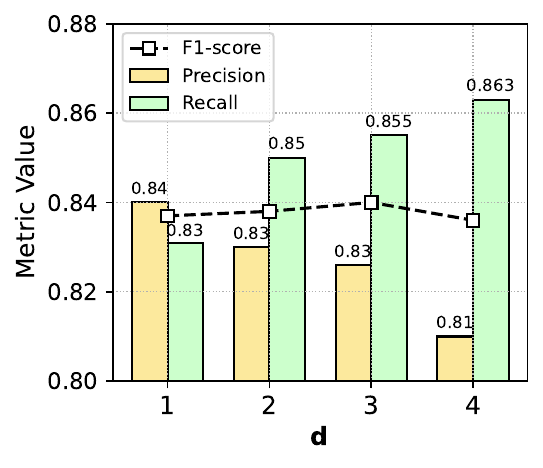}
    \caption{Pattern quality with varying $d$.}
    \label{exp:f1-d}
    
    \vspace{-2mm}
\end{figure}

In the \blue{next examination}, we examine the sensitivity to the accuracy of offline trajectory recovery algorithms. Specically, we inject noise into the recovered trajectories to control the degree of IDF1~\cite{ristani2016performance}, which is a popular metric for the accuracy of multi-object tracking. In terms of noise injection operators, we randomly perform  trajectory shift, node deletion and ID switches to construct degraded trajectories. Specifically, the trajectory shift operator offsets the values of time frames and bounding boxes in a segment of the tracked trajectory, the node deletion operator removes a specific tracked segment, and the ID switch operator replaces the moving object of a segment of the tracked trajectory with another random object.

Figure~\ref{exp:f1-idf1} presents the comparison of the patterns discovered by VPlatoon and VConvoy with the ground truth at different accuracies of input trajectories. We observe that VPlatoon achieves a clear advantage in $F_1$-score for both low and high accuracy trajectories. Our proposed relaxed definition of co-movement pattern can identify more high-quality patterns while introducing negligible irrelevant results.

\begin{figure}[h!]
    \vspace{-2mm}
    \setlength{\abovecaptionskip}{6pt}
    \subfigcapskip=-9pt
    \subfigbottomskip=-3pt
    
    \centering
    \includegraphics[width=0.28\textwidth]{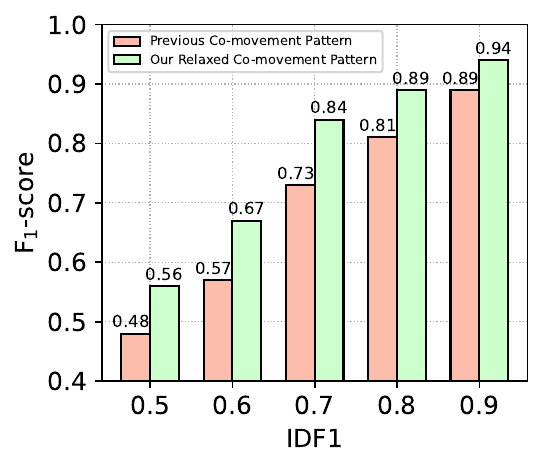}
    \caption{Pattern quality with varying accuries of recovered trajectories.}
    \label{exp:f1-idf1}
    
    \vspace{-2mm}
\end{figure}

\subsection{\blue{Case Study}}
\blue{Finally, We present a case study to further demonstrate the effectiveness of VPlatoon. Two representative VPlatoon patterns from the CityFlow dataset are visualized as illustration. The pattern parameters are set as $m = 2$, $k = 3$, $\epsilon = 12$s, and $d = 2$.}

\blue{Figure~\ref{exp:case-study-1} illustrates a VPlatoon pattern from videos with substantial object occlusion. Due to traffic congestion, vehicle $o_{449}$ and $o_{487}$ are largely obstructed at camera $c_{29}$. As a result, the recovered time intervals of them at camera $c_{29}$ are $[114s, 120s]$ and $[115s, 129s]$, both much shorter than $o_{486}$'s $[101s, 117s]$. Similarly, $o_{486}$ is also misidentified at camera $c_{33}$. These imperfections in the recovered trajectories prevent VConvoy from detecting the co-movement pattern of these three objects.}

\begin{figure}[h!]
    \vspace{-1mm}
    \setlength{\abovecaptionskip}{6pt}
    \subfigcapskip=-9pt
    \subfigbottomskip=-3pt
    
    \centering
    \includegraphics[width=0.4\textwidth]{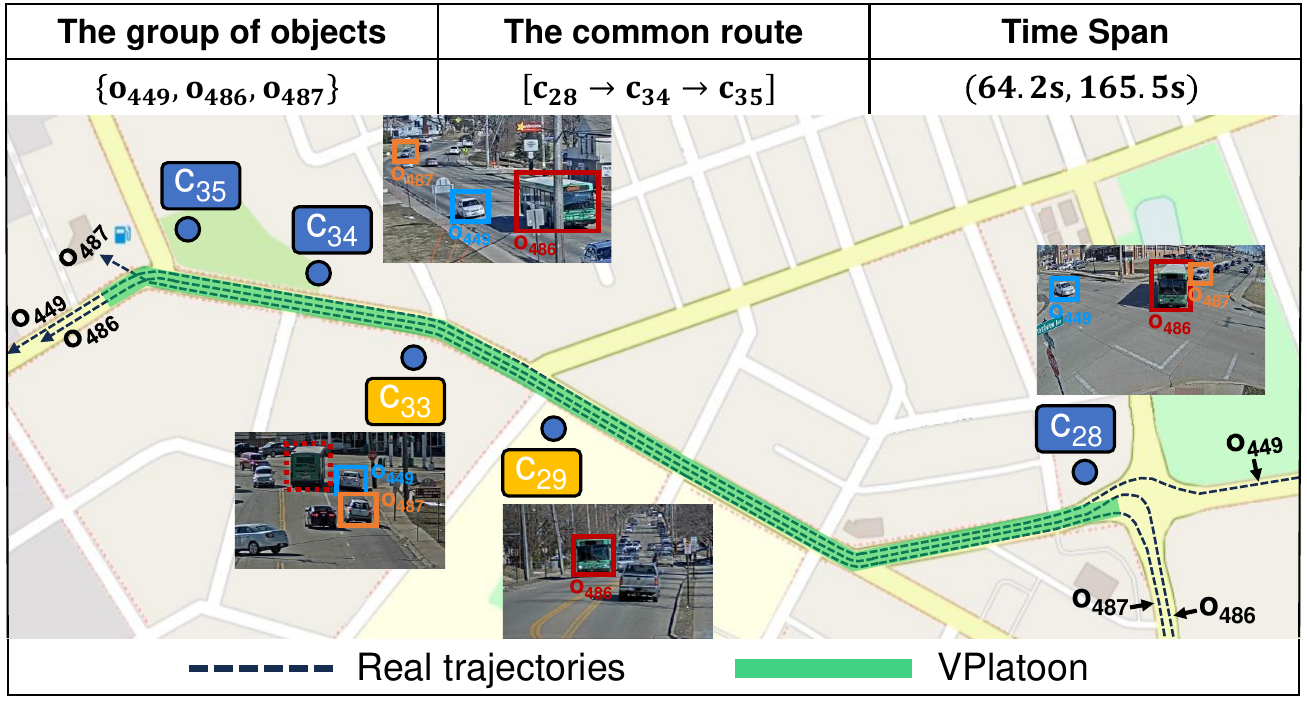}
    \caption{The visualization of a VPlatoon pattern from videos with object occlusion.}
    \label{exp:case-study-1}
    
    \vspace{-2mm}
\end{figure}

\blue{We further examine a VPlatoon pattern from video-recovered trajectories with ID switching in Figure~\ref{exp:case-study-2}. Owing to the differing viewpoints, $o_{398}$ experiences ID switches at both cameras $c_{19}$ and $c_{21}$, causing VConvoy to detect two separate patterns. In contrast, VPlatoon correctly identifies that these two vehicles maintain close motion along the entire route between cameras $c_{16}$ and $c_{25}$, confirming its effectiveness once again.}

\begin{figure}[h!]
    \vspace{-1mm}
    \setlength{\abovecaptionskip}{6pt}
    \subfigcapskip=-9pt
    \subfigbottomskip=-3pt
    
    \centering
    \includegraphics[width=0.4\textwidth]{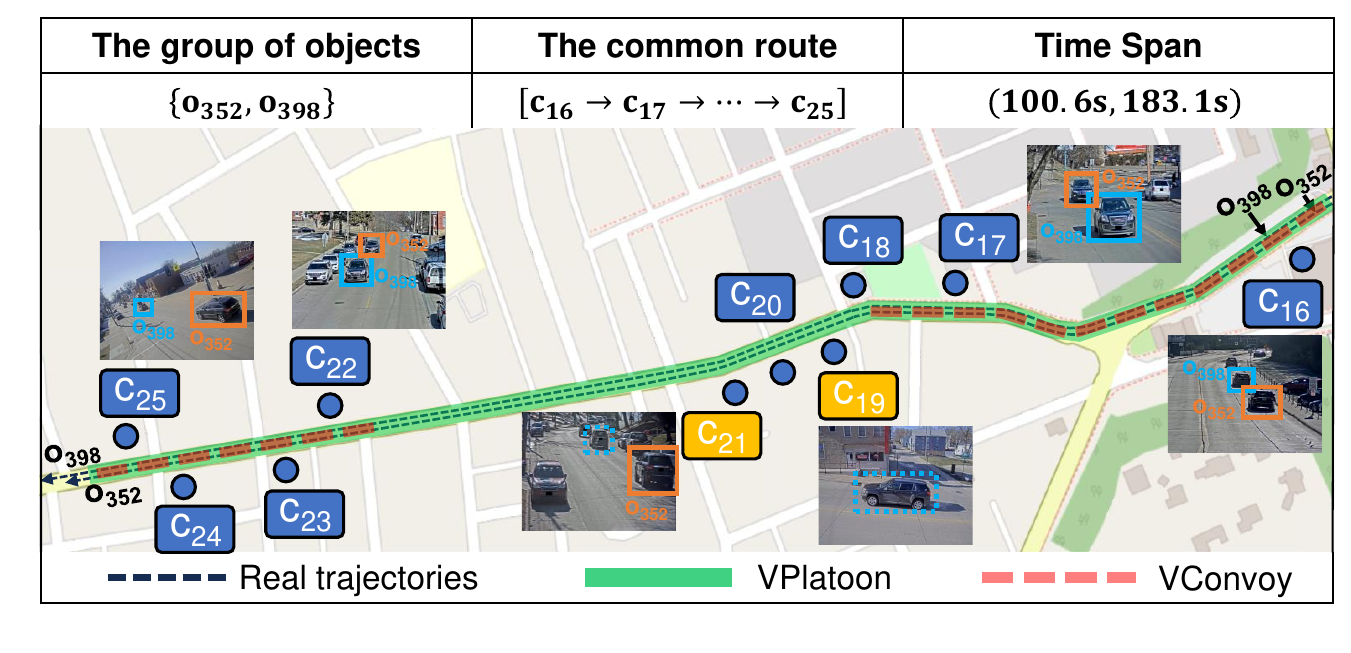}
    \caption{The visualization of a VPlatoon pattern from videos with vehicle mis-matching.}
    \label{exp:case-study-2}
    
    \vspace{-2mm}
\end{figure}

\section{Conclusion} \label{sec:conclusion}
In this paper, we propose a relaxed definition of co-movement pattern from surveillance videos, which uncovers more comprehensive interesting patterns in inaccurately recovered trajectories. We devise a baseline based on previous video-based co-movement pattern mining. A novel enumeration framework called MaxGrowth is also proposed. It efficiently identifies all maximal relaxed co-movement patterns through eliminating the requirement for candidate verification and utilizing the further developed two pruning rules. Extensive experiments confirm the efficiency of MaxGrowth and the effectiveness of our proposed pattern definition.

In the future, it would be interesting to leverage probabilistic models to develop more general and robust definitions of co-movement patterns. In addition, distributed optimization for co-movement pattern mining from video data could also be a promising research direction.


\clearpage

\balance
\bibliographystyle{ACM-Reference-Format}
\bibliography{references}

\end{document}